\documentclass{article}

\usepackage[nonatbib,preprint]{nips_2018}

\usepackage[utf8]{inputenc} 
\usepackage[T1]{fontenc}    
\usepackage{hyperref}       
\usepackage{url}            
\usepackage{booktabs}       
\usepackage{amsfonts}       
\usepackage{nicefrac}       
\usepackage{microtype}      
\usepackage{amsmath,amssymb,amsthm,cite}
\usepackage{graphicx, wrapfig}

\title{Topological Data Analysis of Decision Boundaries with Application to Model Selection}

\author{
  Karthikeyan Natesan Ramamurthy\\
  IBM Research\\
  Yorktown Heights, NY 10598 \\
  \texttt{knatesa@us.ibm.com} \\
  \And
  Kush R.\ Varshney \\
  IBM Research \\
  Yorktown Heights, NY 10598 \\
  \texttt{krvarshn@us.ibm.com} \\
  \And
  Krishnan Mody\thanks{Part of this work was completed when K.\ Mody was with IBM Research.} \\
  New York University \\
  New York, NY 10012 \\
  \texttt{km2718@nyu.edu} \\
}

\newtheorem{theorem}{Theorem}
\newtheorem{lemma}{Lemma}
\theoremstyle{definition}
\newtheorem{definition}{Definition}
\theoremstyle{remark}
\newtheorem*{remark}{Remark}

\begin{document}

\maketitle

\begin{abstract}
We propose the labeled \v{C}ech complex, the plain labeled Vietoris-Rips complex, and the locally scaled labeled Vietoris-Rips complex to perform persistent homology inference of decision boundaries in classification tasks. We provide theoretical conditions and analysis for recovering the homology of a decision boundary from samples.  Our main objective is quantification of deep neural network complexity to enable matching of datasets to pre-trained models; we report results for experiments using MNIST, FashionMNIST, and CIFAR10.
\end{abstract}

\section{Introduction}
\label{sec:intro}

In supervised learning, the term \emph{model selection} usually refers to the process of using validation data to tune hyperparameters.  However, we are moving toward a world in which model selection refers to marketplaces of pre-trained deep learning models in which customers select from a vendor's collection of available models, often without the ability to run validation data through them or being able to change their hyperparameters.  Such a marketplace paradigm is sensible because deep learning models have the ability to generalize from one dataset to another \cite{arpit2017closer, zhang2016understanding, kawaguchi2017generalization}. In the case of classifier selection, the use of data and decision boundary complexity measures, such as the critical sample ratio (the density of data points near the decision boundary), can be a helpful tool \cite{HoB2002,arpit2017closer}. 

In this paper, we propose the use of persistent homology \cite{edelsbrunner2008persistent}, a type of topological data analysis (TDA) \cite{Carlsson2009}, to quantify the complexity of neural network decision boundaries. Persistent homology involves estimating the number of connected components and number of holes of various dimensions that are present in the underlying manifold that data samples come from. This complexity quantification can serve multiple purposes, but we focus how it can be used as an aid for matching vendor pre-trained models to customer data. To this end, we must extend the standard conception of TDA on point clouds of unlabeled data, and develop new techniques to apply TDA to decision boundaries of labeled data. 

In our previous work \cite{VarshneyR2015}, the only prior work we are aware of on TDA of decision boundaries, we use persistent homology to tune hyperparameters of radial basis function kernels and polynomial kernels.  The contributions herein have greater breadth and theoretical depth as we detail below.  A recent preprint also examines TDA of labeled data \cite{GussS2018}, but approaches the problem as standard TDA on separate classes rather than trying to characterize the topology of the decision boundary. In the appendix, we discuss how this approach can be fooled by the internal structure of the classes.  There has also been theoretical work using counts of homological features, known as Betti numbers, to upper and lower bound the number of layers and units of a neural network needed for representing a function \cite{BianchiniS2014}.  That work does not deal with data, as we do here.  Moreover,  its bounds are quite loose and not really usable in practice, similar in their looseness to the bounds for algebraic varieties \cite{Milnor1964, BasuPR2005} cited by \cite{VarshneyR2015} for polynomial kernel machines.

The main steps in a persistent homology analysis are as follows. We treat each data point as a node in a graph, drawing edges between nearby nodes\hspace{1pt---}\hspace{1pt}where nearby is according to a scale parameter. We form complexes from the simplices formed by the nodes and edges, and examine the topology of the complexes as a function of the scale parameter.  The topological features such as connected components, and holes of various dimensions that persist across scales are the ones that capture the underlying shape of the dataset.  In all existing approaches to persistent homology, the scale parameter is a single global value that does not factor in the local scaling of the dataset, making the inference of Betti numbers from persistence brittle and difficult to automate.

Our main contributions are as follows:
\begin{enumerate}
\item We introduce a new simplicial complex construction called the labeled \v{C}ech complex that captures decision boundary topology.  We provide theoretical conditions on the decision boundary and the data samples near the boundary that lead to the successful recovery of the homology of the decision boundary. 
\item We propose a computationally efficient construction of decision boundary surfaces: the labeled Vietoris-Rips complex.  We illustrate the need for local scaling to handle non-uniform sampling of data near the decision boundary and address this need by proposing a simplicial complex construction based on estimates of local scale using a k-nearest neighbors method.
\item We evaluate the merits of the above approaches using synthetic and real-world data experiments. Using synthetic data experiments, we show that the proposed approaches recover the homology even when there is extreme local scaling. Using the real-world application domains MNIST, FashionMNIST and CIFAR10, we show how these approaches can be used to evaluate the topological complexity of decision boundaries of deep neural network classifiers. Our main finding in terms  of model selection can be summarized as follows: when choosing a pre-trained network, one whose topological complexity matches that of the dataset yields good generalization.
\end{enumerate}

We defer detailed background on persistent homology and simplicial constructions for \emph{unlabeled} point cloud data to the appendix.  Throughout this work we assume the labels to be binary for simplicity; multi-class extensions can consider decision boundaries in one-vs-one, one-vs-all and Venn diagram constructions \cite{VarshneyW2010}. 

\section{Labeled \v{C}ech Complex and Recovery Guarantees}
\label{sec:theory}

In this section, we introduce the labeled \v{C}ech (L\v{C}) complex and prove results on its use for recovering the homology of a decision boundary.  The high-level idea is as follows: to recover the homology of a decision boundary, we must cover it such that the cover is its deformation retract.  The practically- and computationally-oriented reader may safely proceed to Section~\ref{sec:top_dec_bound} after noting the definition of the decision boundary and the proposed (computationally intractable) L\v{C} complex.

\subsection{Decision Boundary Manifold}
Decision boundaries are hypersurfaces, surfaces of dimension $d-1$ in ambient spaces of dimension $d$, that divide a space into two classes. We define the overall probability space $\mathcal{Z}$ with the measure given by $\mu_z$ and the pdf $p_Z$. We assume two classes that can be conditioned from this space using the selector $C$; the pdfs being $p_X = p_{Z|C}(z|1)$ and $p_Y = p_{Z|C}(z|0)$. We denote the mixture probabilities as $p_C(0) = q$ and $p_C(1) = 1-q$, such that $p_Z(z) = p_{Z|C} (z|1) p_C(1) + p_{Z|C} (z|0) p_C(0)$. By the Neyman-Pearson rule, the decision boundary manifold is defined by  $\mathcal{M} = \{z \in \mathcal{Z} \mid p_Y = p_X\}$. 

Let us define the extent of the distribution where the two classes are mixed by the set 
\begin{equation}
\mathcal{D} = \{z \in \mathcal{Z} | p_{Z|C}(z|0) > 0, p_{Z|C}(z|1) > 0 \}.
\end{equation} This is the set where both distributions have some mass. We also denote the set $\mathcal{G} = \{z \in \mathcal{Z} \mid (p_{Z|C}(z|0) = 0 \lor p_{Z|C}(z|1) > 0) \land (p_{Z|C}(z|0) > 0 \lor p_{Z|C}(z|1) = 0) \}$, where one of the classes has zero mass and the other class has non-zero mass.

\subsection{Labeled \v{C}ech Complex}

The homology of a manifold can be recovered by an appropriate random sampling and computing a \v{C}ech complex on the random samples. The same idea can be extended to the case of a decision boundary, which is a manifold at the intersection of the two classes. We need a construction which is homotopic to this manifold. To this end, we introduce the \textit{labeled \v{C}ech complex}.
 
\begin{definition}
\label{def:lc_complex}
An $(\epsilon, \gamma)$-labeled \v{C}ech complex, is a simplicial complex with a collection of simplices such that each simplex $\sigma$ is formed on the points in the set $S$ aided by the reference set $W$, when the following conditions are satisfied:
\begin{enumerate}
\item $\displaystyle \bigcap_{s_i \in \sigma} B_{\epsilon}(s_i) \neq \emptyset$, where $s_i$ are the vertices of $\sigma$.
\item $\forall s_i \in \sigma, \quad \exists w \in W$ such that, $\|s_i - w\| \leq \gamma$.
\end{enumerate}
\end{definition}
This definition matches the usual \v{C}ech complex, but introduces the additional constraint that a simplex is induced only if all its vertices are close to some point in the reference set $W$. The second condition also implies that $W$ is $\gamma$-dense in the vertices of the simplices of the $(\epsilon, \gamma)$-L\v{C} complex.

\subsection{Recovery Guarantees}

Now, we derive sufficient sampling conditions so that the L\v{C} complex is homotopic to the decision boundary manifold and hence recovers it homology. The general idea is that when sufficient samples are drawn near $\mathcal{M}$, we can cover $\mathcal{M}$ using balls of radius $r$ ($B_r(z)$), and $U$ deformation retracts to $\mathcal{M}$. The nerve of the covering will be homotopic to $\mathcal{M}$ according to the Nerve Lemma \cite{borsuk1948imbedding}. The intuition is that when we have dense enough sampling, the nerve of the \v{C}ech complex is homotopic to the manifold \cite{niyogi2008finding}. If the sampling is not sufficiently dense, we run into the danger of breaching the `tubular neighborhood' of the manifold since the $\epsilon$ in the \v{C}ech complex has to be large. In our L\v{C} complex, points from one class will be used to construct the actual complex, and the points from the other class will be used as the reference set per Definition \ref{def:lc_complex}.

\paragraph{Sketch of the theory:} Lemma \ref{lem:lc_complex} shows the equivalence of the L\v{C} complex to a particular \v{C}ech complex on unlabeled data, helping us build our theory from existing results in  \cite{niyogi2008finding}. Theorem \ref{thm:set_density} lower bounds the sample size needed to cover two sets of sets, laying the ground for our main sample complexity result. Theorem \ref{thm:density_cover_manifold} provides the sample complexity for a dense sampling of the decision boundary manifold, and the main result in Theorem \ref{thm:sample_complexity} gives the sufficient conditions under which an L\v{C} complex on the sampled points from the two classes will be homotopic to the decision boundary.

Let us assume that the decision boundary is a manifold $\mathcal{M}$ with condition number $1/\tau$. This means that the open normal bundle about $\mathcal{M}$ of radius $\tau$ is embedded in $\mathbb{R}^d$. In other words, the normal bundle is non self-intersecting.  We also place the following assumptions.
\begin{itemize}

\item $\mathcal{D}$ is contained in the tubular neighborhood of radius $r$ around $\mathcal{M}$, i.e., $\mathcal{D} \subset \text{Tub}_r(\mathcal{M})$.
\item For every $0 < s < r$, the mass around a point $p$ in $\mathcal{M}$ is at least $k_{s}^{(c)}$ in both classes. There is sufficient mass in both classes:
\begin{equation}
\label{eqn:reg_prop}
\inf_{p \in \mathcal{M}} \mu_c(B_{\epsilon} (p)) > k_{s}^{(c)} \quad \forall c \in \{0,1\}. 
\end{equation}
\end{itemize}

\begin{lemma}
\label{lem:lc_complex}
As $\epsilon$ varies from $0$ to $\infty$, a filtration is induced on the $(\epsilon, \gamma)$-L\v{C} complex for a fixed $\gamma$.
\end{lemma}
\begin{proof}
Fixing $\gamma$, we choose $S_{\gamma} \subseteq S$, such that $W$ is $\gamma$-dense in $S_{\gamma}$. Therefore, the $(\epsilon, \gamma)$-L\v{C} complex on $V$ is equivalent to an $\epsilon$-\v{C}ech complex on $S_{\gamma}$, and hence varying $\epsilon$ induces a filtration.
\end{proof}

\begin{remark}
The $(\epsilon, \gamma)$-\v{C}ech complex can be used to delineate the decision boundary by choosing $S$ to be the samples of one class and $W$ to be the other class.

Given sufficient samples in $S$ and $W$, $\epsilon$-balls on $S_{\gamma}$ will be homotopic to $\mathcal{M}$. Since homotopy implies same homology, this is how we use the L\v{C} complex to identify the homology of the decision boundary.
\end{remark}

\begin{theorem}
\label{thm:set_density}
Let $\{A_i\}_{i=1}^{l_a}$ and $\{B_j\}_{j=1}^{l_b}$ be two sets of measurable sets. Let $\mu_x$ and $\mu_y$ be the probability measures on $\bigcup_{i=1}^{l_a} A_i$ and $\bigcup_{j=1}^{l_b} B_j$, respectively, such that $\mu_x(A_i) > \alpha_x, \forall i \in \{1, 2, \ldots, l_a\}$ and $\mu_y(B_i) > \alpha_y, \forall j \in \{1, 2, \ldots, l_b\}$. Let $\mu_x$ and $\mu_y$ be the component measures of $\mu_z$, such that $\mu_z(F) = q \mu_x(F)+(1-q) \mu_y(F)$, $q$ and $1-q$ being the mixture probabilities. Let $\overline{z} = \{z_1, z_2, \ldots, z_n\}$ be the set of $n$ i.i.d. draws according to $\mu_z$, which can be partitioned into two sets $\overline{x}$ and $\overline{y}$ which contain the samples from the measures $\mu_x$ and $\mu_y$. Then, if
\begin{equation}
\label{eqn:n_bound_thm}
n \geq \max \left( \frac{1}{\alpha_x q}\left( \log 2 l_a + \log \frac{1}{\delta} \right), 
                            \frac{1}{\alpha_y (1-q)}\left( \log 2 l_b + \log \frac{1}{\delta} \right) \right)
\end{equation} we are guaranteed with probability greater than $1-\delta$ that
\begin{equation}
\label{eqn:intersect_cond_them}
\forall i, \overline{x} \cap A_i \neq \emptyset  \quad \text{and} \quad \forall j, \overline{y} \cap B_j \neq \emptyset.
\end{equation} 
\end{theorem}

\begin{proof}
Let us assume that among the $n$ samples $\overline{z}$ drawn, $|\overline{x}| = n_x$ and $|\overline{y}| = n_y$, so that $n = n_x + n_y$. Let us denote the event $x \notin A_i$ for any $i$ as $E_i^{a}$ and the event $y \notin B_j$ for any $j$ as $E_j^{b}$. The probability of these events are
\begin{align}
P(E_i^{a} | |\overline{x}| = n_x) &= (1-\mu_x(A_i))^{n_x} \leq (1-\alpha_x)^{n_x} \text{ and }\\
P(E_j^{b} | |\overline{y}| = n_y) &= (1-\mu_y(B_i))^{n_y} \leq (1-\alpha_y)^{n_y}.
\end{align} The probability bound on the composite event  (\ref{eqn:intersect_cond_them}) is expressed as
\begin{equation}
P \left( \left( \cap_i\overline{E_i^{a}}\right) \cap \left( \cap_j\overline{E_j^{b}}\right) \right) > 1- \delta,
\end{equation} which simplifies to
\begin{equation}
P \left( \overline{\left( \cup_i E_i^{a}\right) \cup \left( \cup_j E_j^{b}\right)} \right) > 1- \delta.
\end{equation} This implies that 
\begin{equation}
\label{eqn:E_a_b_prob}
P(\cup_i E_i^{a})+P(\cup_i E_j^{b})
\end{equation} should be bounded from above by $\delta$. The individual conditional probabilities can be union-bounded as $P(\cup_i E_i^{a} | |\overline{x}| = n_x) \leq l_a (1-\alpha_x)^{n_x}$ and $P(\cup_i E_j^{b} | |\overline{y}| = n_y) \leq l_b (1-\alpha_y)^{n_y}$. Hence, the upper bound on \eqref{eqn:E_a_b_prob} is
\begin{equation}
\label{eqn:E_a_b_prob_bound}
\sum_{n_x = 0}^n P(\cup_i E_i^{a} | |\overline{x}| = n_x) p(|\overline{x}| = n_x) + P(\cup_i E_j^{b} | |\overline{y}| = n-n_x) p(|\overline{y}| = n-n_x),
\end{equation} which, after some algebra, simplifies to
\begin{equation}
\label{eqn:E_a_b_prob_bound1}
 l_a (1-q \alpha_x)^{n} +  l_b (1-(1-q) \alpha_y)^{n}.
 \end{equation}
We need to find an $n$ such that the expression in \eqref{eqn:E_a_b_prob_bound1} is bounded above by $\delta$.
 
Since we know $1-\alpha q \leq \exp(-\alpha q)$ using Taylor approximation, when $n > \frac{1}{\alpha q } (\log 2l + \log (\frac{1}{\delta}))$, $l \exp(-\alpha q n) \leq \delta/2$. Hence if we pick $n$ according to \eqref{eqn:n_bound_thm}, \eqref{eqn:E_a_b_prob} will be $ \leq \delta$, and with probability greater than $1-\delta$, we can ensure \eqref{eqn:intersect_cond_them}.
\end{proof}

\begin{lemma}
\label{lem:density_correspondence}
For three sets $S$, $W$, and $U$, if $S$ is $r$-dense in $U$ and $W$ is $t$-dense in $U$, there exists an $\hat{S} \subseteq S$, such that the following hold:
\begin{enumerate}
\item \label{lem:c1} $\hat{S}$ is $r$-dense in $U$,
\item \label{lem:c2}  $U$ is $r$-dense in $\hat{S}$,
\item \label{lem:c3}  $W$ is $(r+t)$-dense in $\hat{S}$.
\end{enumerate}
\end{lemma}

\begin{proof}
If $S$ is $r$-dense in $U$, for every $u \in U$, there exists an $s \in S$ such that $\|u-s\| < r$. Now, let  $\hat{S} \subseteq S$, $\hat{S} = \{s \in S \mid \|u-s\| < r, u \in U\}$, i.e., for each element $\hat{s} \in \hat{S}$, we have at least one $u \in U$ such that $\|u-\hat{s}\| < r$ and vice-versa. This proves item \ref{lem:c1} and item \ref{lem:c2}. Since for each $u$, we have at least one $w \in W$, such that $\|u-w\| < t$. Hence, by the triangle inequality, for each $\hat{s} \in \hat{S}$, we have at least one $w \in W$ such that $\|s-w\| < (r+t)$.
\end{proof}

\begin{theorem}
\label{thm:density_cover_manifold}
Let $N_{r/2}$ and $N_{s/2}$ be the $r/2$ and $s/2$ covering numbers of the manifold $\mathcal{M}$. Let $G$ and $H$ be two sets of points in $\mathcal{M}$ of sizes $N_{r/2}$ and $N_{s/2}$ such that $B_{r/2}(g_i), g_i \in G$, and $B_{s/2}(h_j), h_j \in H$ are the $r/2$- and $s/2$-covers. Let $\overline{z}$ be generated by i.i.d.\ sampling from $\mu_z$ whose two component measures satisfy the regularity properties in \eqref{eqn:reg_prop}, and have mixing probabilities $q$ and $1-q$ for $q > 0$. Let the two component samples be $\overline{x}$ and $\overline{y}$. Then if 
$\displaystyle |\overline{z}| > \max\left(\frac{1}{q k_{r/2}^{(0)}}  \left(\log \left(2 N_{r/2} \right)+ \log \left(\frac{1}{\delta} \right)\right),  \frac{1}{(1-q) k_{s/2}^{(0)}}  \left( \log \left( 2 N_{r/2} \right)+ \log \left( \frac{1}{\delta} \right) \right) \right)$,
 with probability greater than $1-\delta$, $\overline{x}$ will be $r$-dense in $\mathcal{M}$, and $\overline{y}$ will be $s$-dense in $\mathcal{M}$.
\end{theorem}

\begin{proof}
Letting $A_i = B_{r/2}(g_i)$, and $B_j = B_{s/2}(h_j)$, apply the previous Theorem. Hence, with probability greater than $1-\delta$, each of  $A_i$ and $B_j$ are occupied by at least one of $x_i \in \overline{x}$, and $y_j \in \overline{y}$ respectively. There it follows that for any $p \in \mathcal{M}$, there is at least one $x \in \overline{x}$ and $y \in \overline{y}$ such that $\|p-x\| < r$, and $\|p-y\| < s$. Thus, with high probability, $\overline{x}$ is $r$-dense in $\mathcal{M}$ and $\overline{y}$ is $s$-dense in $\mathcal{M}$.
\end{proof}

Now we extend Theorem 7.1 in \cite{niyogi2008finding} to the case of the L\v{C} complex and provide the main conditions under which the homology of the decision boundary can be recovered.

\begin{theorem}
\label{thm:sample_complexity}
Let $N_{r/2}$ and $N_{s/2}$ be the $r/2$ and $s/2$ covering numbers of the submanifold $\mathcal{M}$ of $\mathbb{R}^N$. 
Let $\overline{z}$ be generated by i.i.d.\ sampling from $\mu_z$ whose two component measures satisfy the regularity properties in (\ref{eqn:reg_prop}), and have mixing probabilities $q$ and $1-q$ for $q > 0$. Let the two component samples be $\overline{x}$ and $\overline{y}$. Then if 
$\displaystyle |\overline{z}| > \max\left(\frac{1}{q k_{r/2}^{(0)}}  \left(\log \left(2 N_{r/2} \right)+ \log \left(\frac{1}{\delta} \right)\right),  \frac{1}{(1-q) k_{s/2}^{(0)}}  \left( \log \left( 2 N_{s/2} \right)+ \log \left( \frac{1}{\delta} \right) \right) \right)$,
 with probability greater than $1-\delta$, the $(\epsilon, r+s)$-L\v{C} complex will be homotopic to $\mathcal{M}$, if: (a) $r < (\sqrt{9}-\sqrt{8}) \tau$, and (b) $\epsilon \in \left( \frac{(r+\tau)+ \sqrt{r^2+\tau^2-6\tau r}}{2}, \frac{(r+\tau)+ \sqrt{r^2+\tau^2-6\tau r}}{2} \right)$.
 \end{theorem}
 
 \begin{proof}
 From Lemma \ref{lem:density_correspondence}, we know that when $\overline{x}$ is $r$-dense in $\mathcal{M}$, and $\overline{y}$ is $s$-dense in $\mathcal{M}$, we have $\tilde{x} \subseteq \overline{x}$ which is also $r$-dense in $\mathcal{M}$ and $\overline{y}$ is $(r+s)$-dense in $\tilde{x}$. Also, from Lemma \ref{lem:lc_complex}, the $(\epsilon, r+s)$-L\v{C} complex on $\overline{x}$ with the reference set $\overline{y}$ is equivalent to the $\epsilon$-\v{C}ech complex on $\tilde{x}$.
 
 Since $\tilde{x}$ is $r$-dense on $\mathcal{M}$, it follows from Theorem 7.1 in \cite{niyogi2008finding} that this $\epsilon$-\v{C}ech on $\tilde{x}$ will be homotopic to $\mathcal{M}$ if the conditions on $r$ and $\epsilon$ are satisfied. 
 \end{proof}

\section{Labeled Vietoris-Rips Complexes}
\label{sec:top_dec_bound}

In this section, we propose two computationally-tractable constructions for simplicial complexes of the decision boundary: one we name the plain labeled Vietoris-Rips complex and the other we name the locally scaled labeled Vietoris-Rips complex.  We illustrate the need for the locally scaled version.

\subsection{Notation}

Let us start with a labeled discrete sample $\{(z_1, c_1), \ldots, (z_n, c_n)\}$ where $z \in \mathbb{R}^d$ is the data point and $c \in \{0, 1\}$ is the class label. Given a data point $z_i$, we define its neighborhood as $\mathcal{N}_{\theta}(z_i)$ where $\theta$ is a scalar neighborhood parameter. The neighbors are restricted to data points whose class $c_j$ is not the same as $c_i$. Our neighborhood construction is symmetric by definition, hence $z_j \in \mathcal{N}_{\theta}(z_i) \Leftrightarrow z_j \in \mathcal{N}_{\theta}(z_i)$. This results in a bipartite graph $G_{\theta}$. 

We use Betti numbers to describe the topology of the decision boundary. $\beta_i$ is the $i^\text{th}$ Betti number: the number of homology groups $H_i$ of dimension $i$. 

\subsection{Two Complexes}

To induce a simplicial complex with simplices of order greater than one from the bipartite graph $G$, we  connect all $2$-hop neighbors. Since the original edges are only between points in opposing classes, all $2$-hop neighbors belong to the same class. A pictorial depiction of this is provided in Appendix \ref{sec:higher_order_simplices}.  This new graph is defined to be one-skeleton of the decision boundary complex. We create a simplicial complex from this one-skeleton using the standard Vietoris-Rips induction \cite{zomorodian2010fast}: a simplex of dimension $r+1$ is inductively included in the complex if all its $r$-dimensional faces are included. We call this the labeled Vietoris-Rips (LVR) complex $\mathcal{V}_{\theta}$.

Our construction is such that, by definition, for $\theta_2 \geq \theta_1$, there is an inclusion $G_{\theta_2} \supseteq G_{\theta_1}$.  Given this inclusion relationship in the bipartite graphs, we obtain a filtration as we vary $\theta$, i.e., for $\theta_2 \geq \theta_1$, $\mathcal{V}_{\theta_2} \supseteq \mathcal{V}_{\theta_1}$. We provide two approaches for creating the LVR complex and its filtration.
\paragraph{Plain LVR (P-LVR) Complex:} We set $\theta$ to be the radius parameter $\epsilon$ and define $\mathcal{N}_{\theta}(z_i)$ as the set of points $\{z_j\}_{c_j \neq c_i, \|z_i - z_j\| \leq \theta}$. 

\paragraph{Locally Scaled LVR (LS-LVR) Complex:} We set $\theta$ to be $\kappa$, the multiplier to the local scale and define $\mathcal{N}_{\theta}(z_i)$ as the set of points $\{z_j\}_{c_j \neq c_i, \|z_i - z_j\| \leq \kappa \sqrt{\rho_i \rho_j}}$, where $\rho_i$ is the local scale of $z_i$. This is defined to be the radius of the smallest sphere centered at $z_i$ that encloses at least $k$ points from the opposite class. LS-LVR construction is based on the generalization of CkNN  graph introduced in \cite{berry2016consistent} to labeled data.

After the LVR filtrations have been obtained, persistent homology of the decision boundaries can be estimated using standard approaches \cite{edelsbrunner2008persistent, zomorodian2005computing}, and represented using barcodes or persistence diagrams (PDs) \cite{edelsbrunner2012persistent}.

\subsection{Illustration of Homology Group Recovery}
\label{sec:demo_homology_rec}
\begin{wrapfigure}{R}{3.3cm}
\caption{A 2-class data with \emph{red} and \emph{blue} classes (top), and the LS-LVR decision boundary complex at $\kappa = 1.005$ (bottom).}\label{fig:two_cir}
\includegraphics[width=3.1cm]{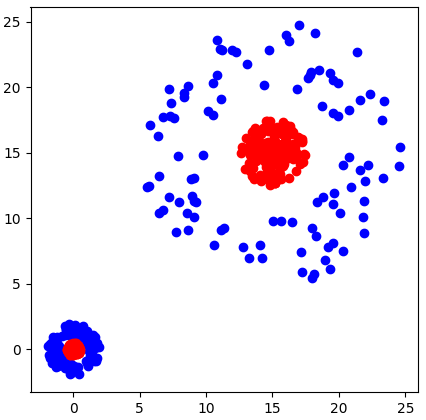}

\includegraphics[width=3.1cm]{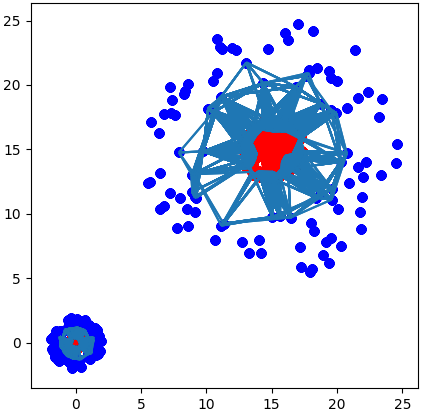}
\vspace{-20pt}
\end{wrapfigure} 

We illustrate these two approaches for constructing decision boundary complexes and estimating their persistent homology using a two-dimensional, two-class dataset given in Figure \ref{fig:two_cir} (top).  The two decision boundaries are homotopic to two circles that separate the classes, and hence the true Betti numbers of the decision boundaries for this data are: $\beta_0 = 2$, and $\beta_1 = 2$. The sampling is non-uniform, with the smaller disk and annulus having more density than the larger ones.

We compute the persistent homology using the P-LVR and LS-LVR complexes. With P-LVR, we vary the radius parameter $\epsilon$ from $0$ to $10$, and with LS-LVR, we vary the local scale multiplier $\kappa$ from $0.5$ to $1.5$. The local scale $\rho$ is computed with $k=5$ neighbors. Figure \ref{fig:two_cir} (bottom) shows a LS-LVR complex at scale $1.005$ that accurately recovers the Betti numbers of the decision boundary.

Figure \ref{fig:pers_two_circle} shows the PDs as well as the Betti numbers for different scales using the two complexes.\footnote{Note that the PD for $H_0$ groups shows all the groups, whereas the Betti numbers in Figures \ref{fig:pers_two_circle}(b) and \ref{fig:pers_two_circle}(d) only show the number of non-trivial homology groups. Non-trivial $H_0$ groups are defined to be those that have more than one data point i.e., the number of simply connected components with size more than 1. Including trivial homology groups is meaningless when computing the topology of decision boundaries since decision boundaries are defined only across classes.} The LS-LVR construction recovers both $\beta_0$ and $\beta_1$ accurately for $\kappa$ slightly greater than 1 and persists until $\kappa$ is slightly less than 1.2. Around this value, one of the holes closes and a little later the other hole collapses as well. The resulting two simply connected components persist until $\kappa=1.5$. 

In contrast, for the P-LVR complex, the $H_0$ and $H_1$ groups first come to life at $\epsilon = 0.9$ for the smaller decision boundary component. The $H_1$ group vanishes almost immediately.  At $\epsilon = 0.38$, the $H_0$ and $H_1$ groups for the larger decision boundary component come to life, persisting for $0.12$. The overall topology ($\beta_0 = 2, \beta_1 = 2$) is not captured at any one scale due to varying sizes of homological features as well as non-uniform sampling. The widely varying life times for homological features make it hard to choose a threshold for estimating the correct number of homology groups. This is not a problem with LS-LVR since the $H_1$ groups appear clustered together in the PD. Another benefit of LS-LVR is that non-noisy homology groups appear around $\kappa = 1$, the natural local scale of the data. This does not hold true for the P-LVR complex.

The actual complexes for various scales with the two constructions are given in the appendix.

\begin{figure}
\begin{minipage}[c]{0.24\linewidth}
  \centering
  \includegraphics[width=3.4cm]{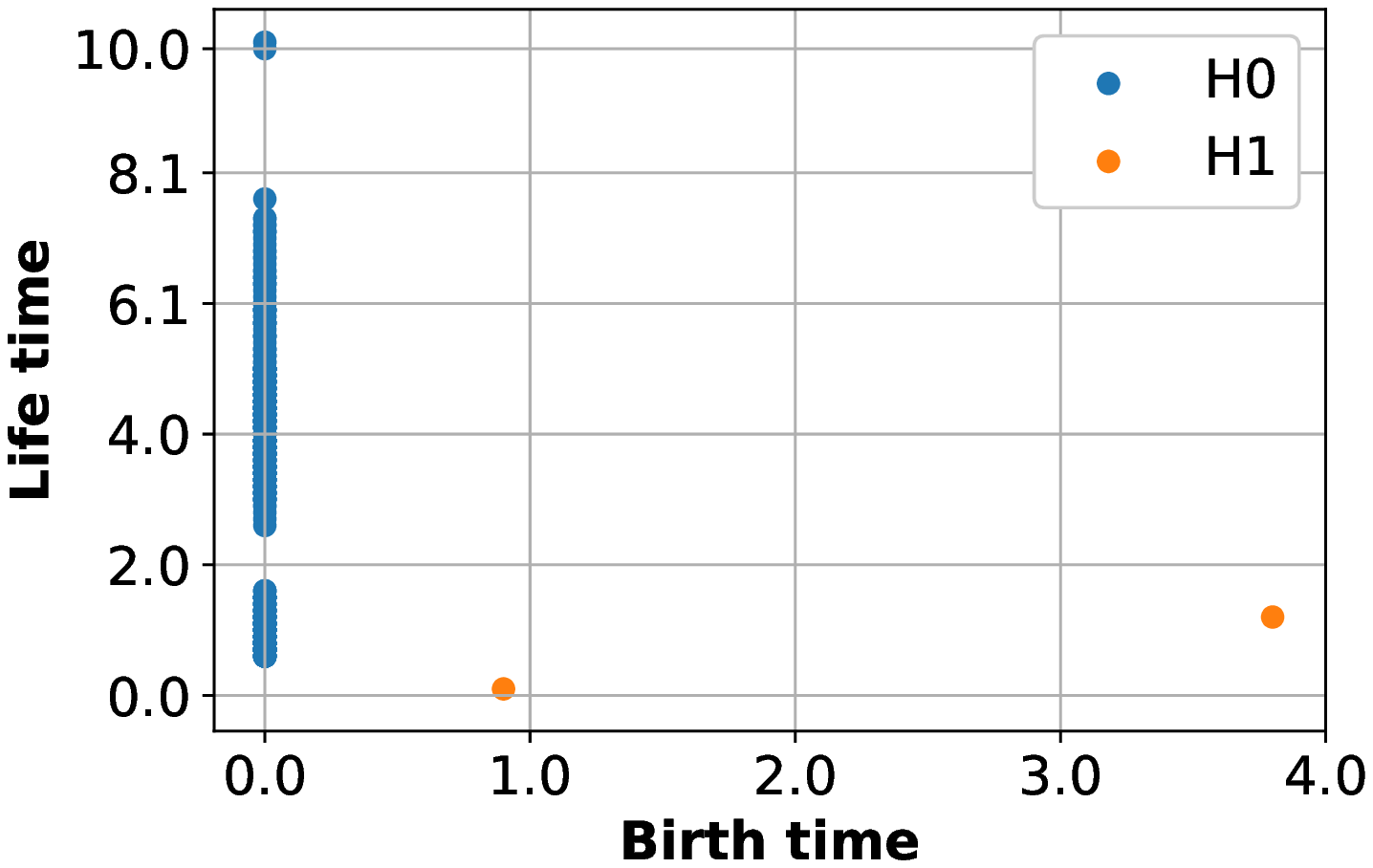}
  \centerline{\footnotesize{(a)}}\medskip
\end{minipage}
\begin{minipage}[c]{0.24\linewidth}
  \centering
  \includegraphics[width=3.4cm]{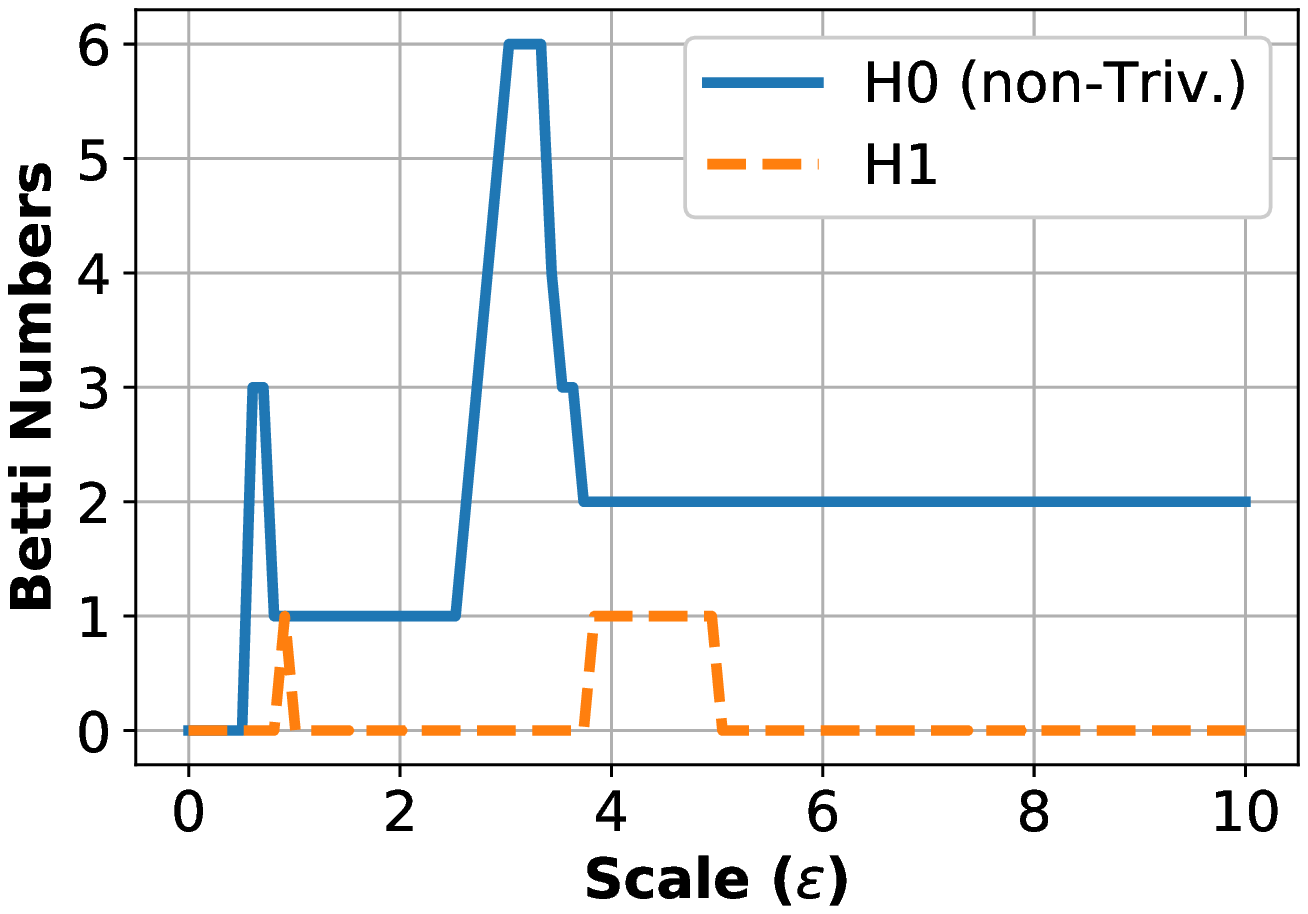}
  \centerline{\footnotesize{(b)}}\medskip
\end{minipage}
\begin{minipage}[c]{0.24\linewidth}
  \centering
  \includegraphics[width=3.4cm]{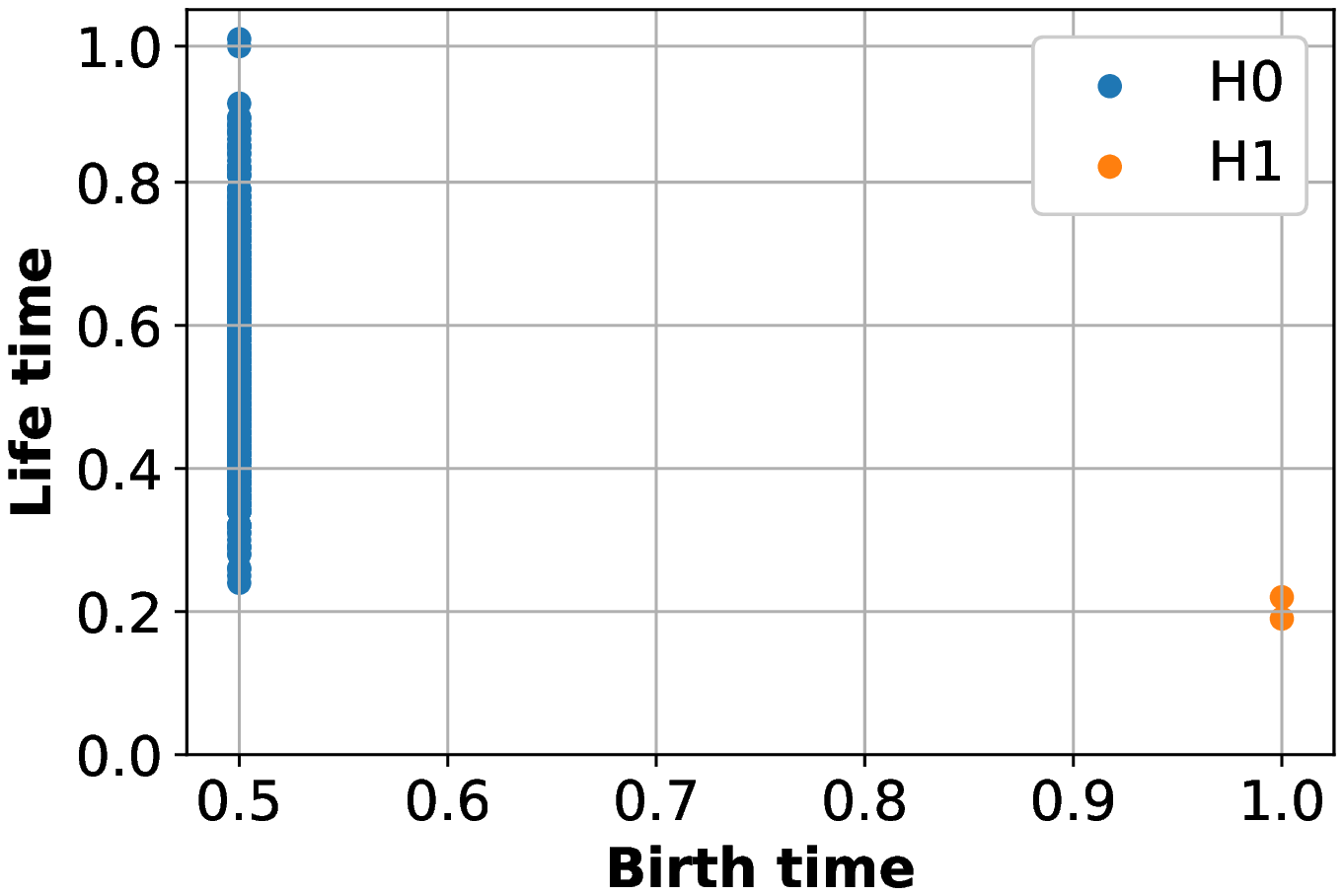}
  \centerline{\footnotesize{(c)}}\medskip
\end{minipage}
\begin{minipage}[c]{0.24\linewidth}
  \centering
  \includegraphics[width=3.4cm]{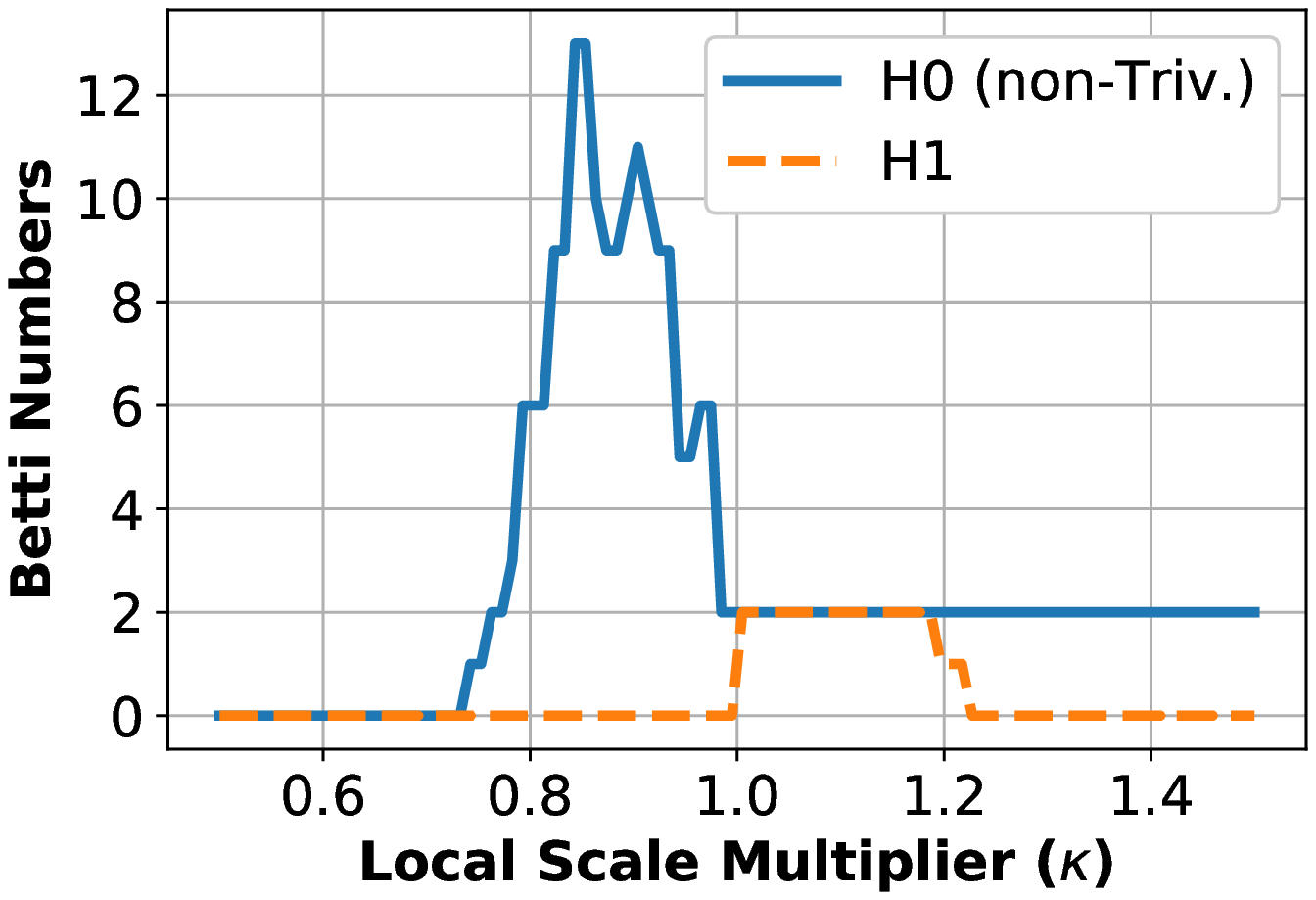}
  \centerline{\footnotesize{(d)}}\medskip
\end{minipage}
\caption{(a) Persistence diagram and (b) Betti numbers as a function of scale using P-LVR, and (c) persistence diagram and (d) Betti numbers using LS-LVR.}
\label{fig:pers_two_circle}
\end{figure}

\section{Experiments}
\label{sec:experiments}

We perform experiments with synthetic and high-dimensional real-world datasets to demonstrate: (a) the effectiveness of our approach in recovering homology groups accurately, and (b) the utility of this method in discovering the topological complexity of neural networks and their potential use in choosing pre-trained models for a new dataset.

In all experiments, to limit the number of simplices, we upper bound the number of neighbors used to compute the neighborhood graph to 20. The results can be reproduced using the code available at \url{https://github.com/nrkarthikeyan/topology-decision-boundaries}. More implementation notes are available in Appendix \ref{sec:impl_notes}.

\subsection{Synthetic Data: Homology Group Recovery}
\label{sec:homo_group_rec_25}
\begin{wrapfigure}{R}{3.3cm}
\caption{A $2-$class data with $\beta_0 = 25, \beta_1 = 25$. Notice the wide variation in sizes of topological features.}\label{fig:25_cir}
\includegraphics[width=3.1cm]{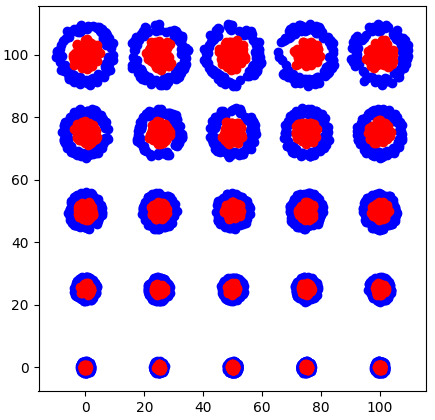}
\vspace{-20pt}
\end{wrapfigure} 
The first experiment demonstrates the effectiveness of our approach in recovering homology groups of complex synthetic data with wide variations in sizes of topological features (Figure \ref{fig:25_cir}). The decision boundary is homotopic to 25 circles ($\beta_0 = 25, \beta_1 = 25$). From Figures \ref{fig:pers_25circles}(c) and \ref{fig:pers_25circles}(d), it is clear that the LS-LVR complex shows similar persistence for all the $25$ $H_1$ groups irrespective of their varying sizes in the dataset. Observe the clumping in the PD, and the presence of a lone noisy $H_1$ group with almost zero life time. The P-LVR complex also recovers the $25$ $H_1$ groups, but does so at different times (Figures \ref{fig:pers_25circles}(a) and \ref{fig:pers_25circles}(b)). From the PD, we can see that there are five rough clumps of $H_1$ groups, around birth times $\{1,2,3,4,5\}$, each containing five $H_1$ groups. The birth times correspond to the radii of the five groups of decision boundaries in Figure \ref{fig:25_cir}. The staggered recovery of topology with the P-LVR complex makes it hard to fix a noise threshold on life times to estimate the correct Betti numbers.

\subsection{Real-World Data: Complexity Estimation and Model Selection}
\label{sec:model_selection}

\begin{figure}
\begin{minipage}[c]{0.24\linewidth}
  \centering
  \includegraphics[width=3.4cm]{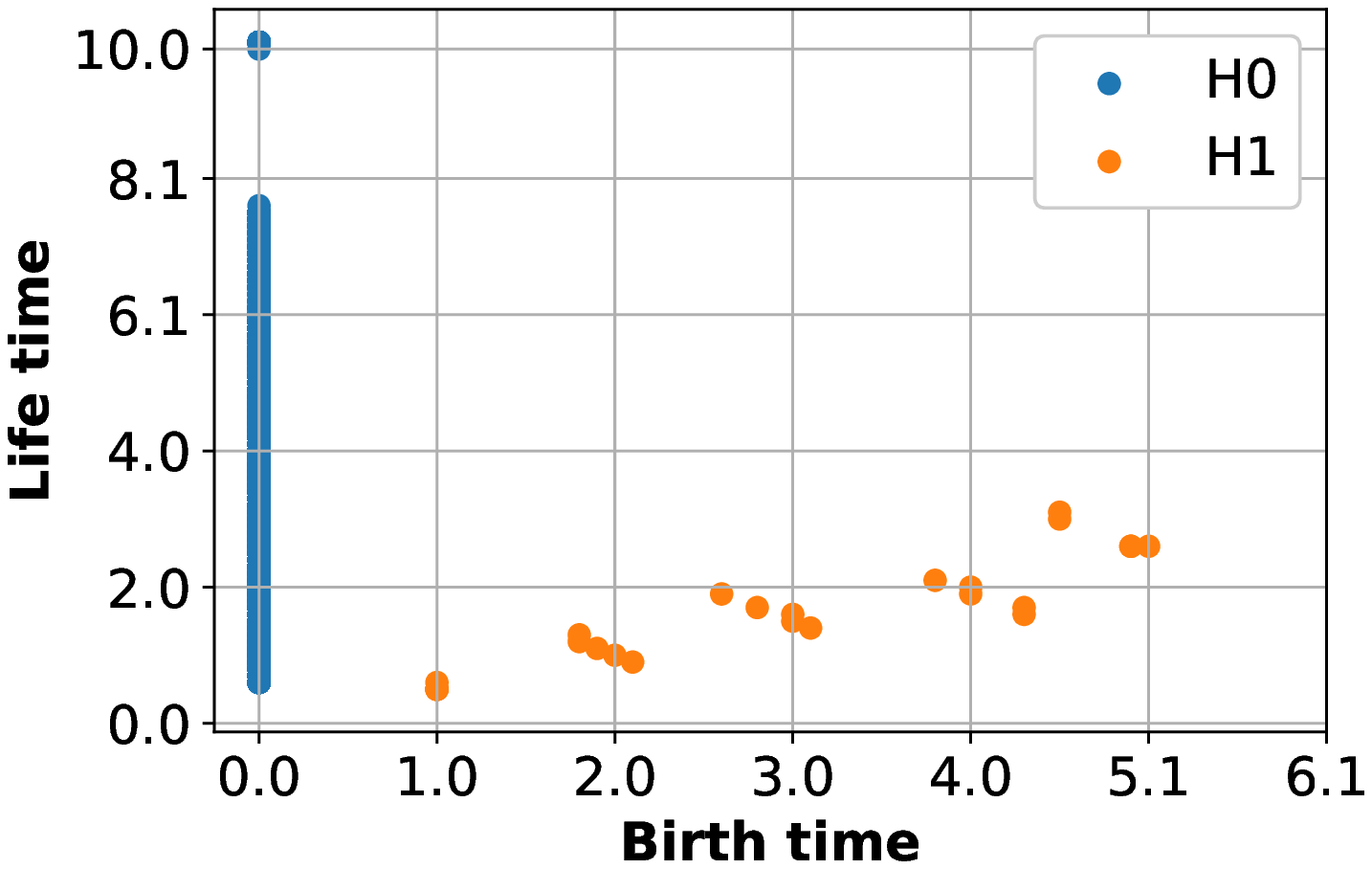}
  \centerline{\footnotesize{(a)}}\medskip
\end{minipage}
\begin{minipage}[c]{0.24\linewidth}
  \centering
  \includegraphics[width=3.4cm]{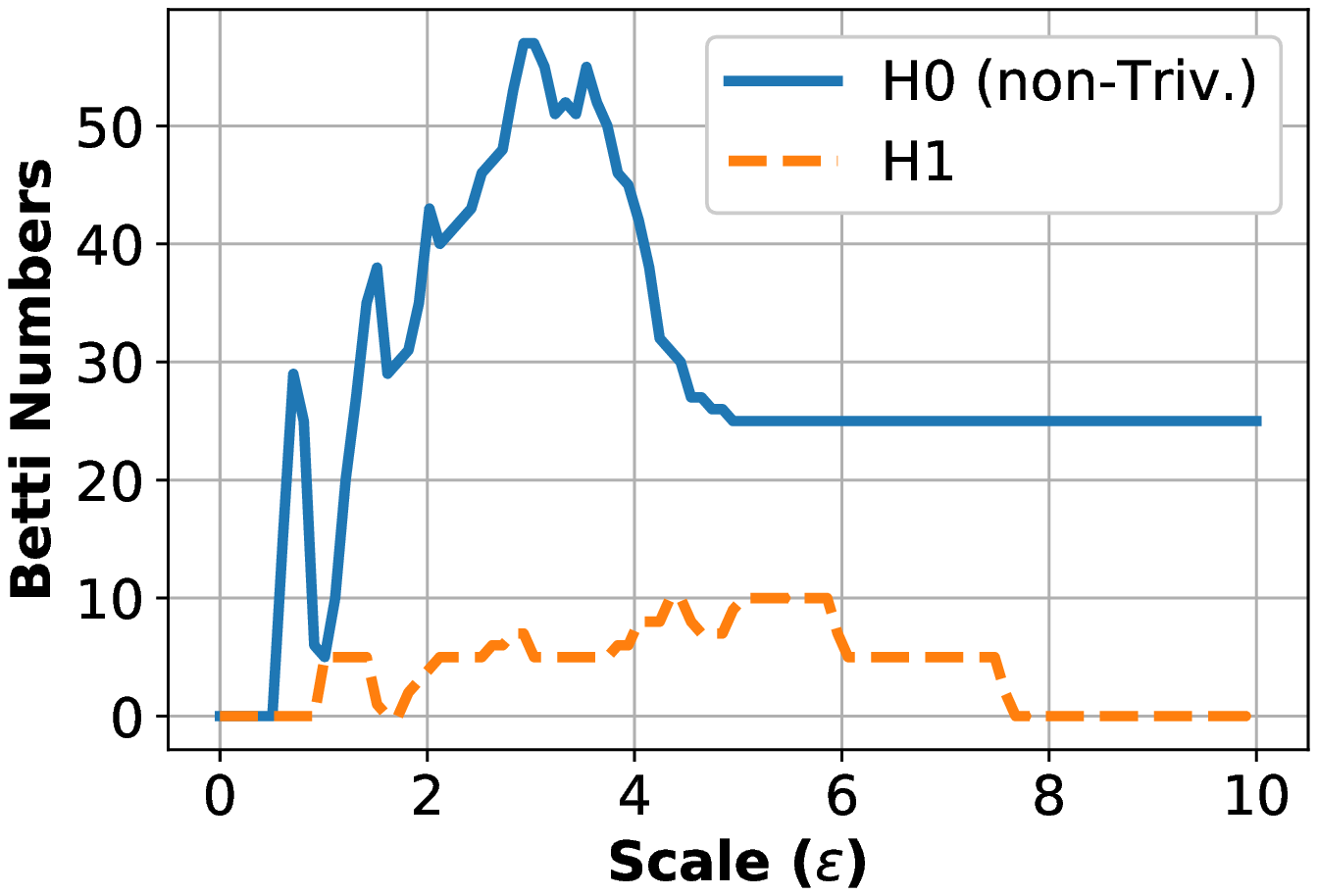}
  \centerline{\footnotesize{(b)}}\medskip
\end{minipage}
\begin{minipage}[c]{0.24\linewidth}
  \centering
  \includegraphics[width=3.4cm]{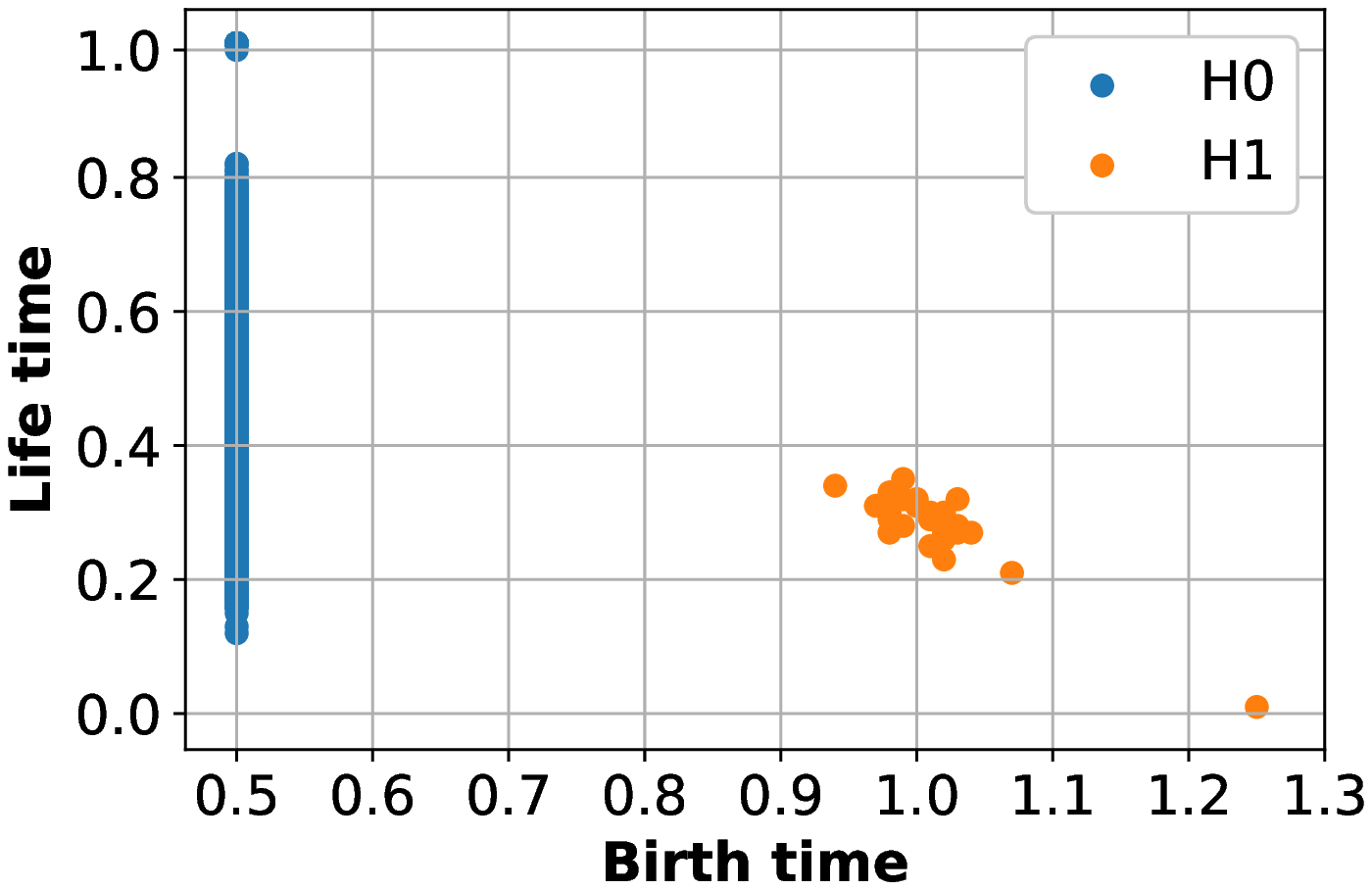}
  \centerline{\footnotesize{(c)}}\medskip
\end{minipage}
\begin{minipage}[c]{0.24\linewidth}
  \centering
  \includegraphics[width=3.4cm]{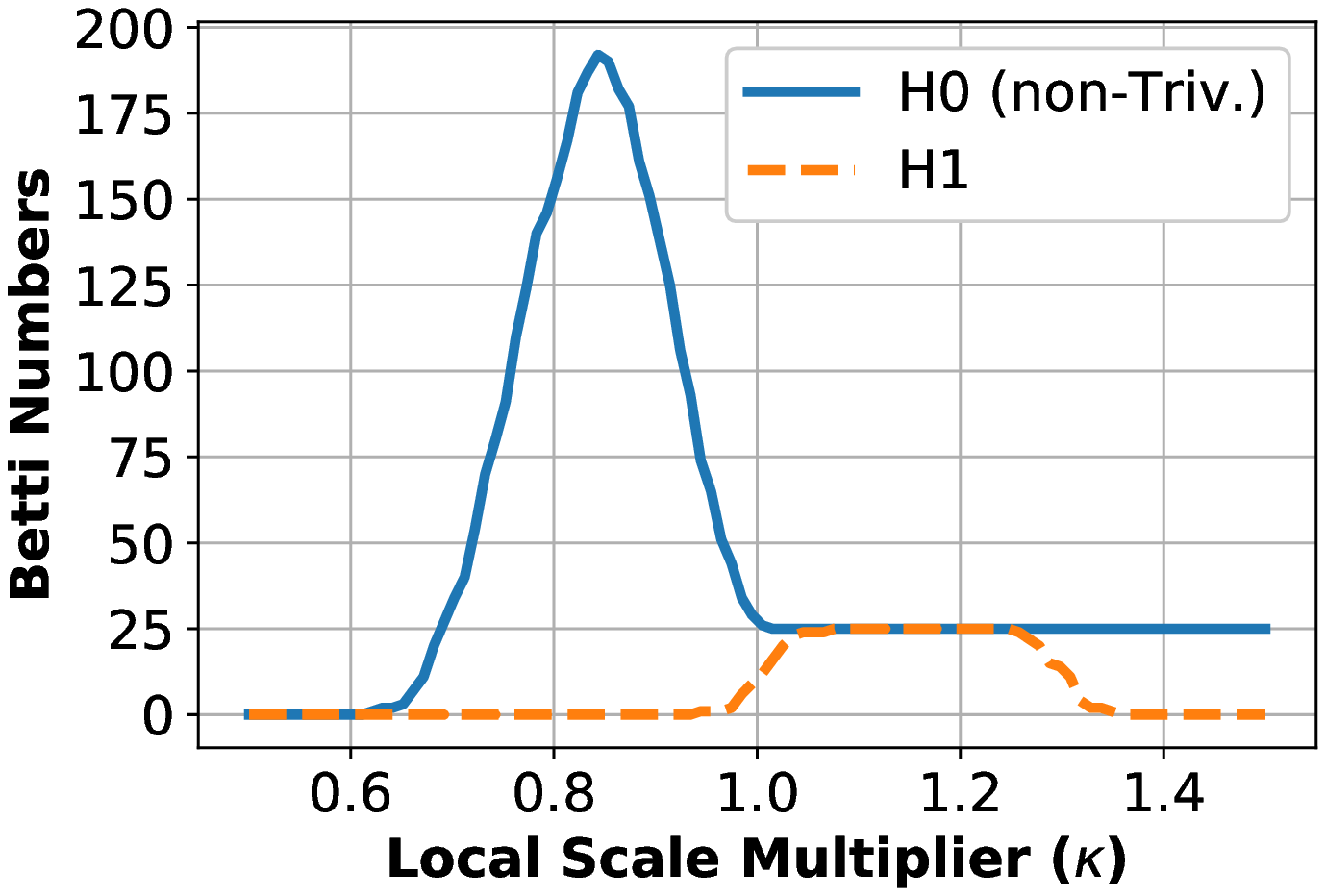}
  \centerline{\footnotesize{(d)}}\medskip
\end{minipage}
\caption{(a) Persistence diagram and (b) Betti numbers as a function of scale using P-LVR, and (c) persistence diagram and (d) Betti numbers using LS-LVR.}
\label{fig:pers_25circles}
\end{figure}

\begin{figure}
\begin{minipage}[c]{0.32\linewidth}
  \centering
  \includegraphics[width=5cm]{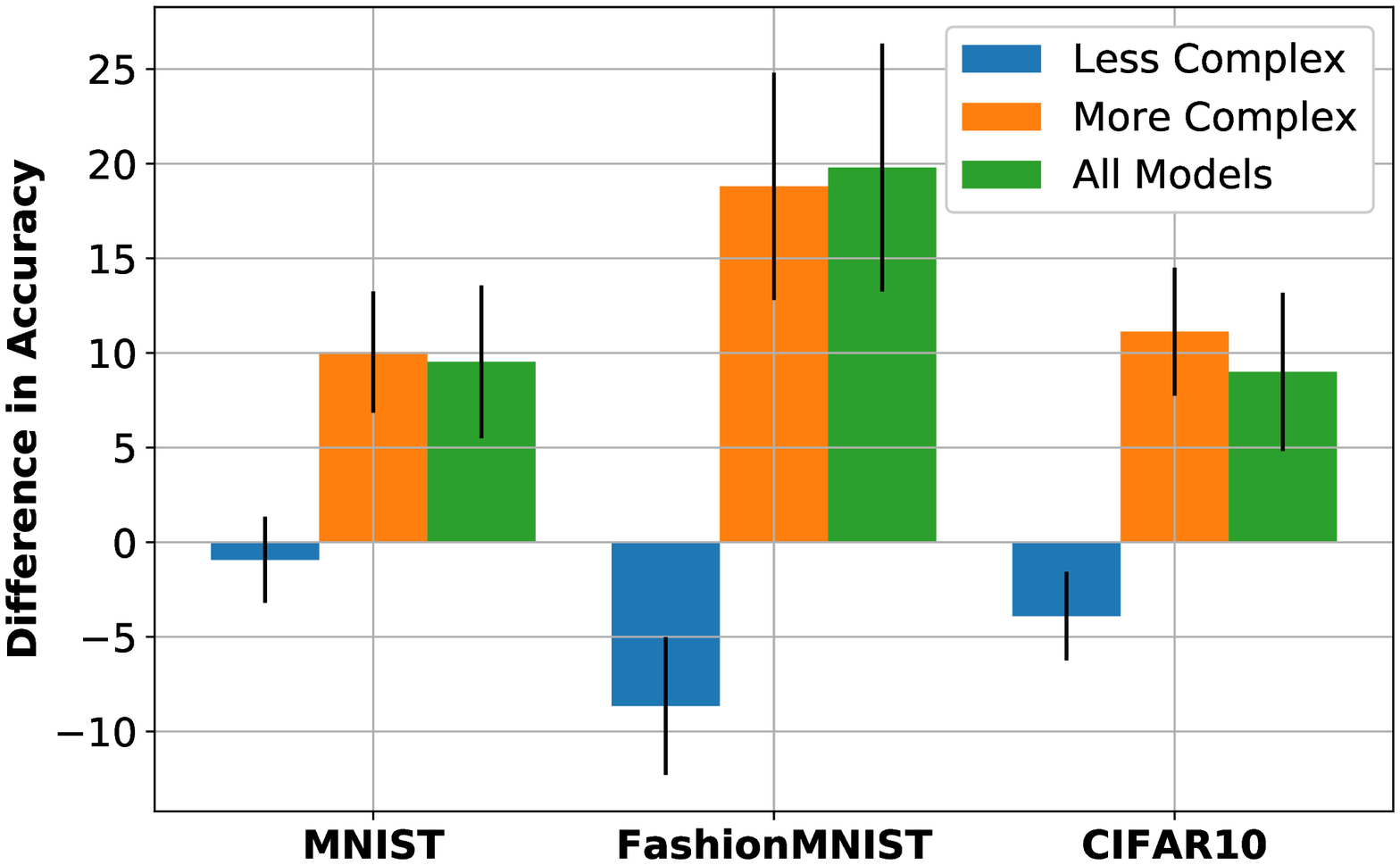}
  \centerline{\footnotesize{(a)}}\medskip
\end{minipage}
\begin{minipage}[c]{0.32\linewidth}
  \centering
  \includegraphics[width=5cm]{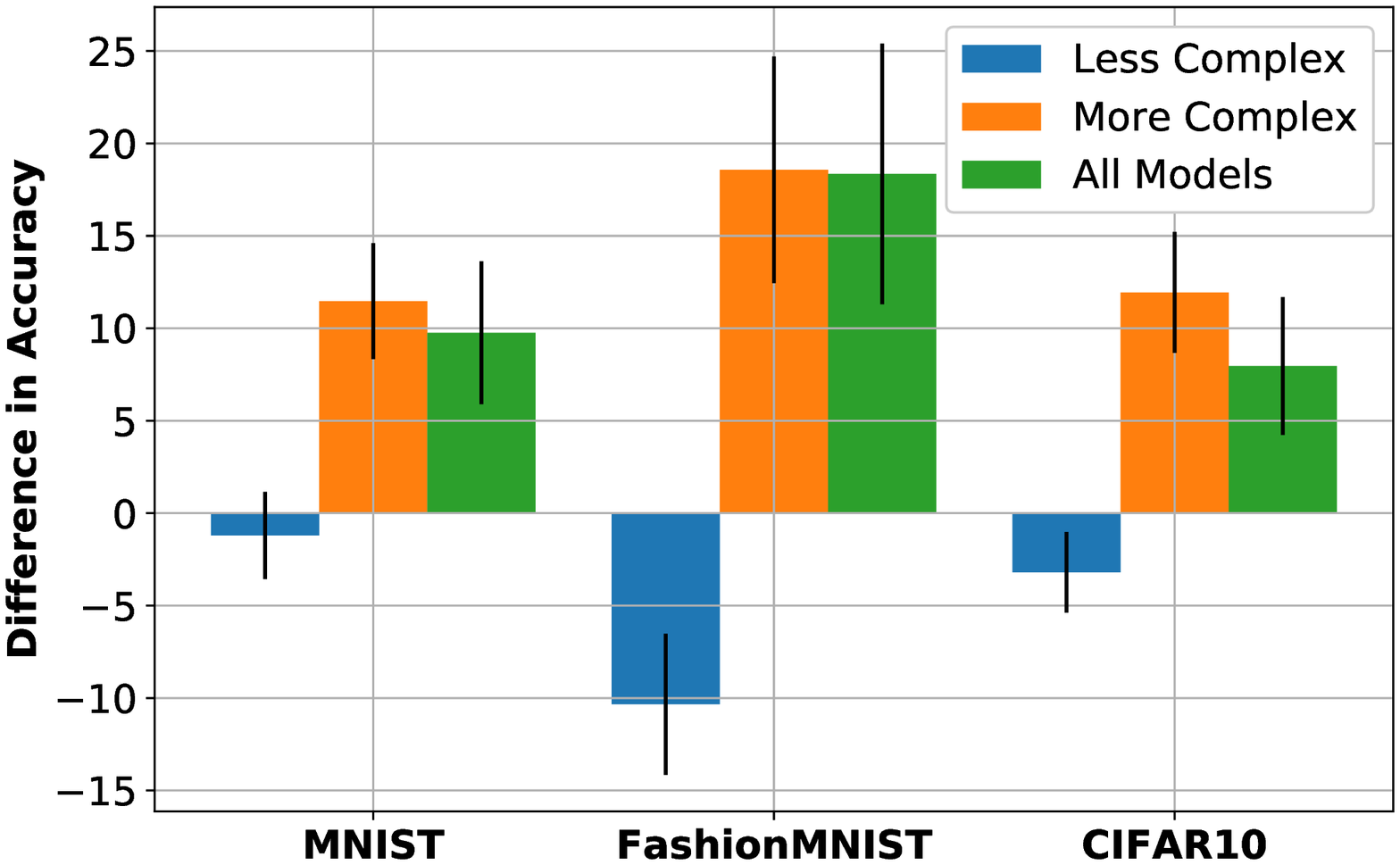}
  \centerline{\footnotesize{(b)}}\medskip
\end{minipage}
\begin{minipage}[c]{0.32\linewidth}
  \centering
  \includegraphics[width=5cm]{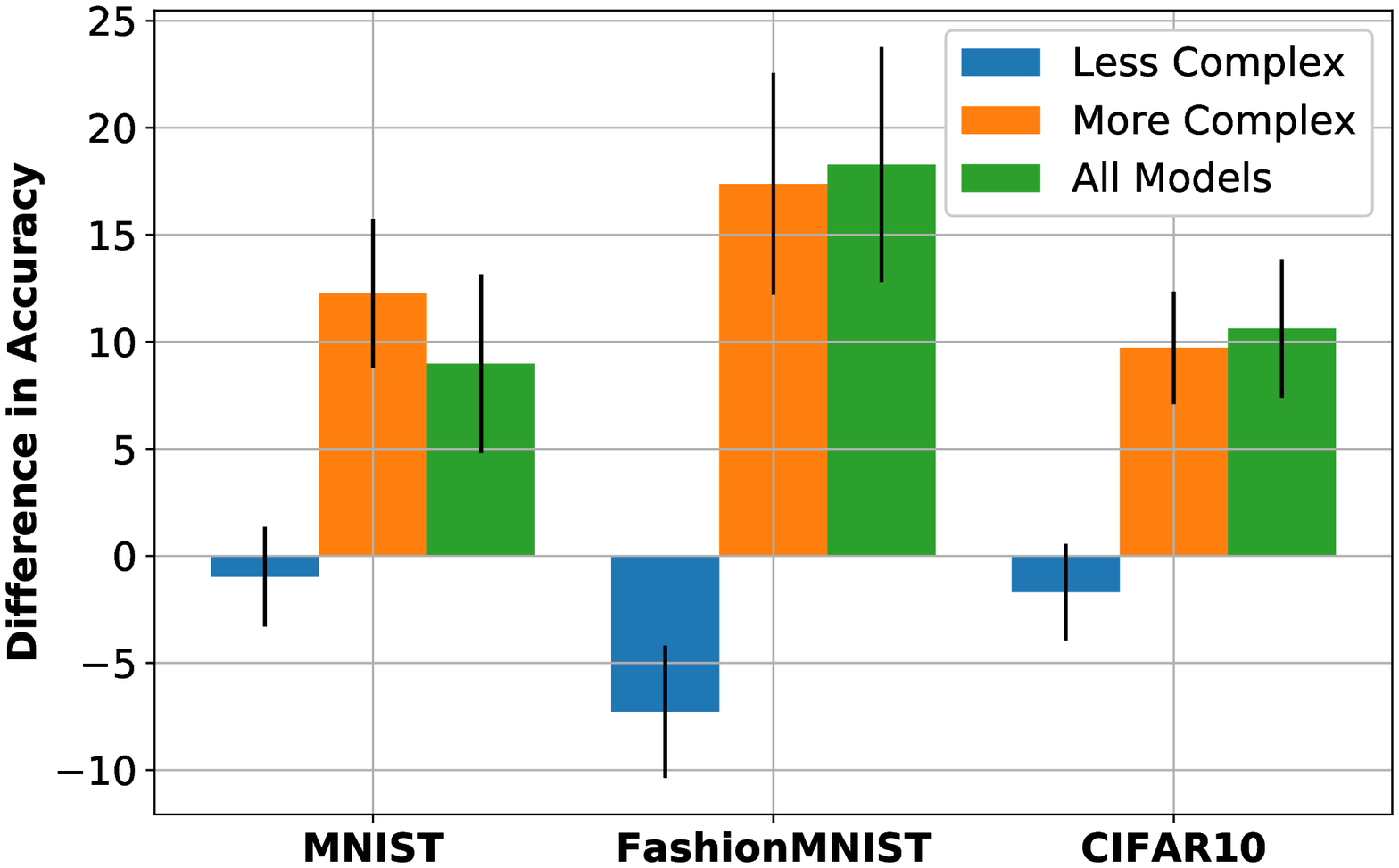}
  \centerline{\footnotesize{(c)}}\medskip
\end{minipage}
\caption{Accuracy improvement or reduction in choosing pre-trained classifiers with topological complexity close to the dataset versus complexity far from the dataset. Complexity measures used: (a) Sum of total lifetimes of $H_0$ and $H_1$ groups, (b) Total lifetimes of $H_0$ groups, (c) Total lifetimes of $H_1$ groups. Blue bars show the accuracy difference when using only pre-trained classifiers with less topological complexity than the dataset, orange bars correspond to those with greater complexity, and green bars correspond to using all pre-trained classifiers. The black lines show the $95\%$ confidence interval.}
\label{fig:topo_comp_acc}
\end{figure}

We demonstrate how topological complexity can be used to guide selecting appropriate pre-trained models for a new dataset. We use only LS-LVR complexes for estimating topological complexities. We consider three application domains for our evaluation: MNIST \cite{mnistlecun}, FashionMNIST \cite{xiao2017fmnist} and CIFAR10 \cite{krizhevsky2009learning}. FashionMNIST is a drop-in replacement for MNIST with the same image sizes and train-test splits. 

All three applications have $10$ classes and $50,000$ training and $10,000$ test images. Each instance of MNIST and FashionMNIST is a $28 \times 28$ grayscale image, whereas each instance of CIFAR10 is a $32 \times 32$ color image. We construct $\binom{10}{2} = 45$ binary classification datasets from each application domain, one for each combination of two classes.  We then train individual binary classifiers for these $45$ datasets per application using the standard CNN architecture provided in \url{https://github.com/pytorch/examples/tree/master/mnist} for MNIST and FashionMNIST, and the VGG CNN  - configuration D for CIFAR10 \cite{simonyan2014very}. 

Given a trained model $f_i(\cdot)$, $i=1,\ldots,45$, we evaluate its topological complexity using the test data inputs and predicted labels. This labeled dataset is represented as $\hat{Z}_i = \{(\hat{z}_{i,1}, f_i(\hat{z}_{i,1})), \ldots, (\hat{z}_{i,n_i}, f_i(\hat{z}_{i,n_i}))\}$, and its Betti numbers for $H_0$ and $H_1$ at scale $\kappa$ are given as $\beta_{0, \kappa}(i), \beta_{1, \kappa}(i)$. The complexity of a novel dataset is estimated using its inputs and true labels, using three different measures. The first is total lifetime $H_0$ groups given by $\sum_{\kappa}\beta_{0, \kappa}$, the second is total lifetime $H_1$ groups, $\sum_{\kappa}\beta_{1, \kappa}$, and the third is sum of the two, $\sum_{\kappa}\beta_{0, \kappa} + \beta_{1, \kappa}$. Although these are natural measures of topological complexity, one can consider other reasonable variants as well. 

Let us consider an example of matching a novel dataset to a pre-trained model.  Our novel dataset is MNIST handwritten digit 0 vs. handwritten digit 4, whose $\beta_0 + \beta_1$ data complexity we compute to be 479.  Then we look for pre-trained model complexities that are similar.  Not surprisingly, the closest is the pre-trained model 0 vs. 4, which has a model complexity 479.  The 0 vs. 9 pre-trained model has a similar complexity of 525.  If we select the 0 vs. 4 model, we achieve 99.95\% accuracy on 0 vs. 4 data, and if we select the 0 vs. 9 model, we also achieve a high accuracy of 96.08\%.  If we select a model that is not well-matched to the data complexity, for example the 0 vs. 5 model with complexity 1058, we achieve a low accuracy on 0 vs. 4 data of 63.41\%. 

All data and model complexities are listed in the appendix. For MNIST, FashionMNIST and CIFAR10, the average binary classification accuracy of the best performing models is $99.61\%$, $98.39\%$, and $96.78\%$ respectively. Now let us conduct an experiment to see whether the example above holds in general.  Treating each of the 45 datasets as the novel dataset, we select $5$ pre-trained models that are the closest and $5$ models that are the farthest in topological complexity. We evaluate these classifiers on the novel dataset and obtain the average difference in classification accuracy between the closest and farthest classifiers. If the difference in accuracy is significantly greater than zero, it means that using classifiers that have similar topological complexity as the dataset is beneficial. If the difference in accuracy is close to zero, it shows that there is no benefit in using topological complexity to guide the choice of the classifier. If it is significantly less than zero, it means that classifiers which do not have similar topological complexity are better suited for the novel dataset.

Armed with this intuition, we can interpret Figure \ref{fig:topo_comp_acc}. The green bars show the average accuracy difference obtained by repeating the above experiment on the $45$ two-class datasets in each of CIFAR10, MNIST and FashionMNIST. The black lines show the $95\%$ confidence interval. If the black line is completely above (below) $0$, with a $p$-value less than $0.05$, the null hypothesis that the accuracy difference is less than or equal to (greater than or equal to) $0$ can be rejected. If the black line intersects $0$, we cannot reject the null hypothesis that the accuracy difference is $0$, at a significance level of $0.05$.

From the green bars, we see that classifiers with similar topological complexity have higher performance on the novel dataset for all three complexity measures. We then divide the pre-trained classifiers into two groups: those that have lower topological complexity than the novel dataset, and those that have higher topological complexity. Results for classifiers with higher topological complexity are shown using orange bars, and the previous claim still holds. For classifiers with lower complexity, the results show a different trend. Note that in this case, the farthest classifiers have less complexity than the closest ones. For MNIST, there is no significant change in accuracy when choosing any classifier that has lower complexity than the dataset. For CIFAR10, there is a small improvement in accuracy when choosing classifiers that have much lower complexity than the dataset, and for FashionMNIST, this improvement is a little higher. For these three application domains, we observe that choosing classifiers with lower complexities than data is slightly favorable or neutral.

\section{Conclusion}
\label{sec:concl}

In this paper, we have investigated the use of topological data analysis in the study of labeled point clouds of data encountered in supervised classification.  In contrast to \cite{GussS2018}, which simply applies known, standard, persistent homology inference methods to different classes of data separately and does not scale to high dimensions, we introduce new techniques and constructions for characterizing \emph{decision boundaries} and apply them to several commonly used datasets in deep learning.  We propose and theoretically analyze the \emph{labeled} \v{C}ech complex, deriving conditions on recovering the decision boundary's homology with high probability based on the number of samples and the condition number of the decision boundary manifold.  

Furthermore, we have proposed the computationally-tractable \emph{labeled} Vietoris-Rips complex and extended it to account for variation in the local scaling of data across a feature space.  We have used this complex to provide a complexity quantification of pre-trained models and datasets that is able to correctly identify the complexity level below which a pre-trained model will suffer in its ability to generalize to a given dataset.  This use has increasing relevance as model marketplaces become the norm.

\appendix

\section{Background on Persistent Homology for Unlabeled Data}
\label{sec:persisthom}
Consider a set of $n$ data points in $\mathbb{R}^d$: $\mathcal{Z} = \{\mathbf{z}_1,\ldots,\mathbf{z}_n\}$.  A set of points by itself has no shape per se, but if the points are viewed as samples from some shape, then the set of points reveals the underlying shape.  We would like to estimate and approximate the topology of that shape by constructing a simplicial complex from the points and examining the topology of the simplicial complex.  A zero-dimensional simplex is a point, a one-dimensional simplex is a line segment, a two-dimensional simplex is a triangle, a three-dimensional simplex is a tetrahedron, and so on; a simplicial complex is a set of simplices glued together in a particular way.  Specifically, a simplicial complex $\mathcal{S}$ = $(\mathcal{Z},\Sigma)$, where $\Sigma$ is a family of non-empty subsets of $\mathcal{Z}$ such that each subset $\sigma \in \Sigma$ is a simplex.  Furthermore, the following condition must also hold: $\sigma \in \Sigma$ and $\tau \subseteq \sigma$ implies that $\tau \in \Sigma$.  In forming these non-empty subsets of points that form a simplex, we only consider subsets of points that are close to each other.  There are various notions of closeness that we come back to later in this section.

Topology, being the study of shape, is primarily concerned with the number of connected components and the number and dimension of holes that an object has.  The Betti numbers characterize the connectivity as follows.  The zeroth Betti number $\beta_0$ is the number of connected components, the first Betti number $\beta_1$ is the number of one-dimensional holes or circles, the second Betti number $\beta_2$ is the number of two-dimensional voids or cavities, and so on.  For example, a torus or inner tube has $\beta_0 = 1$ because it is just one component, $\beta_1 = 2$ because of the main hole through the middle and the hole formed when looking at a cross-section, and $\beta_2 = 1$ because of the cavity of the inner tube.  Betti numbers for simplicial complexes are defined in the same way.  Formally, $\beta_k(\mathcal{S})$ is the dimension of the $k$th homology group of the complex $H_k(\mathcal{S})$ \cite{Carlsson2009}.

Various approaches exist for constructing simplicial complexes from $\mathcal{Z}$. All of these depend on a scale parameter $\epsilon$ (also referred to as \emph{time}) which specifies the extent of closeness of points. In the \v{C}ech complex \v{C}ech$(\mathcal{Z},\epsilon)$, a simplex is created between a set of points $\mathcal{G}$ if and only if there is a non-empty intersection of the closed Euclidean balls $B(\mathbf{z}_i,\epsilon/2)$, $\forall i \in \mathcal{G}$. In the Vietoris-Rips (VR) complex, VR$(\mathcal{Z},\epsilon)$, a simplex is created if and only if the Euclidean distance between every pair of points is less than $\epsilon$. Efficient construction of the VR complex can proceed by  creating an $\epsilon-$neighborhood graph, also referred to as the \emph{one-skeleton} of $\mathcal{S}$.  Then inductively, triplets of edges that form a triangle are taken as two-dimensional simplices, sets of four two-dimensional simplices that form a tetrahedron are taken as three-dimensional simplices, and so on.

Homological inference depends on the scale parameter (time) at which the complexes are constructed. The topological features of the simplicial complex $\mathcal{S}$ constructed from the data points $\mathcal{Z}$ that are stable across scales, i.e.~that are \textit{persistent}, are the ones that provide information about the underlying shape.  Topological features that do not persist are noise.  \emph{Persistence diagrams} are representations of the birth and death times of each homology cycle corresponding to each homology group $H_k$, $k = 0, 1, \ldots$, i.e.\ for increasing values of the scale parameter, the $\epsilon$ value at which a topological feature begins to exist and ceases to exist.  

\begin{figure}
\begin{minipage}[c]{0.32\linewidth}
  \centering
  \includegraphics[width=2.8cm]{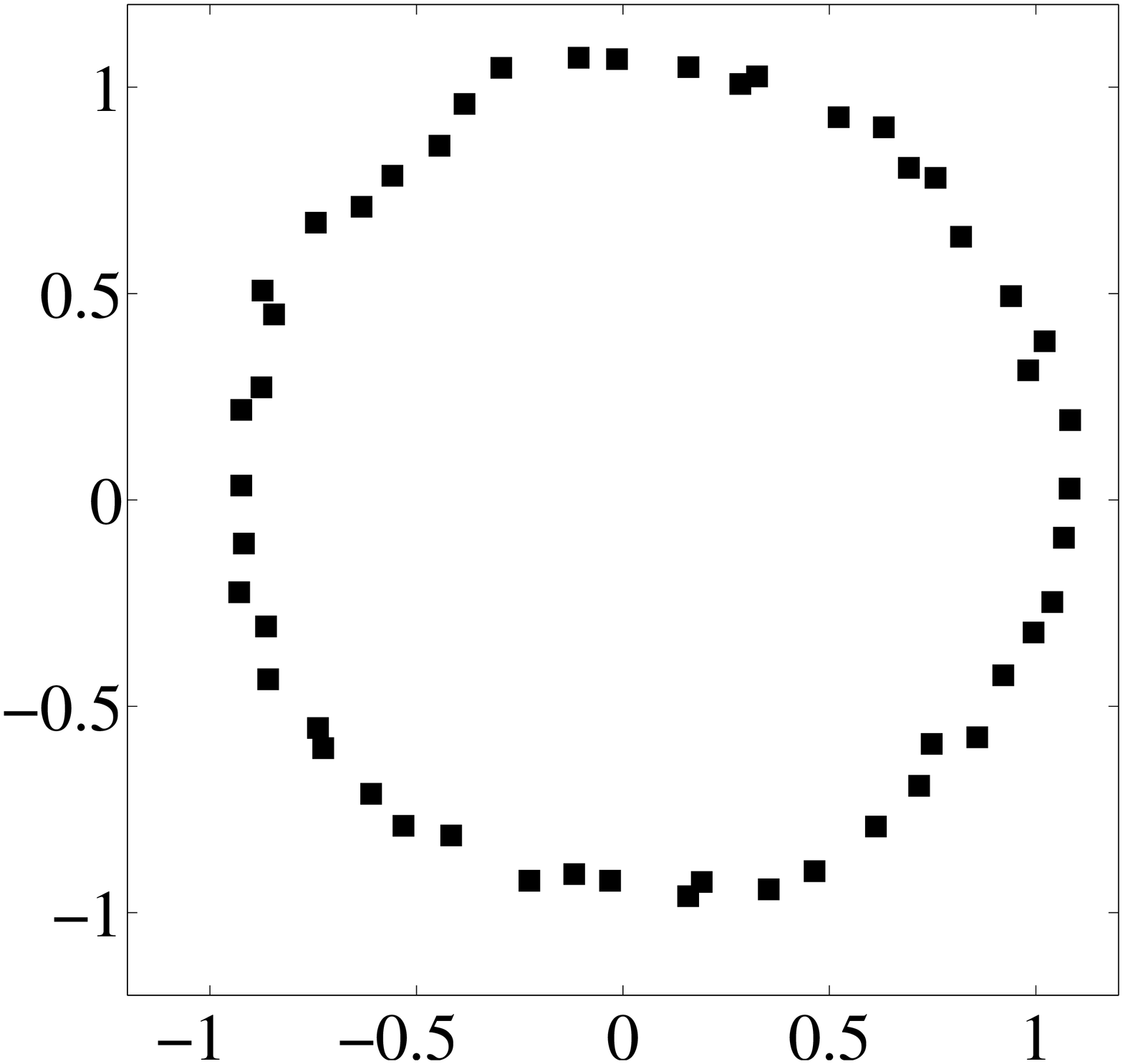}
  \centerline{\footnotesize{(a)}}\medskip
\end{minipage}
\begin{minipage}[c]{0.32\linewidth}
  \centering
  \includegraphics[width=2.8cm]{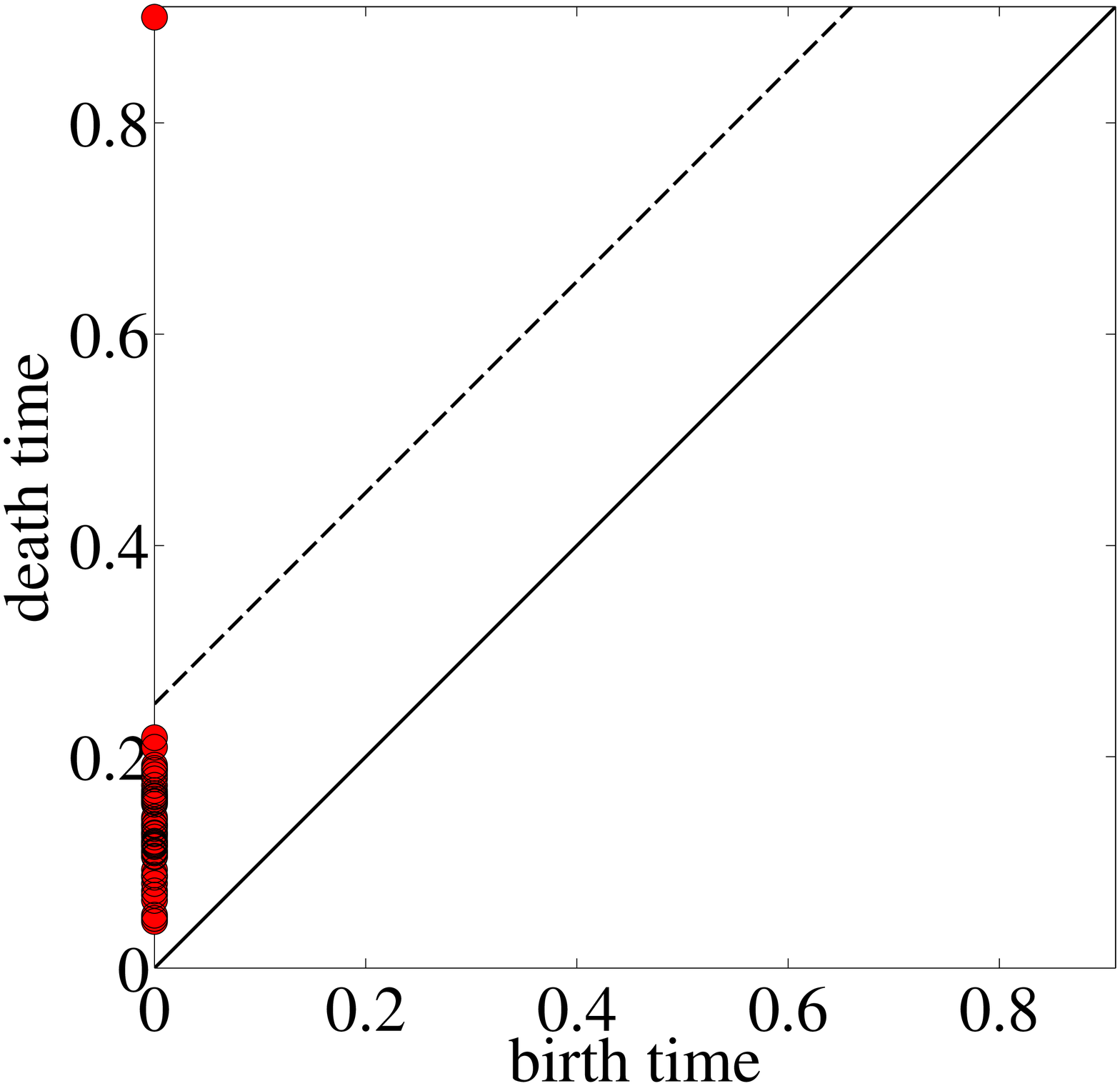}
  \centerline{\footnotesize{(b)}}\medskip
\end{minipage}
\begin{minipage}[c]{0.32\linewidth}
  \centering
  \includegraphics[width=2.8cm]{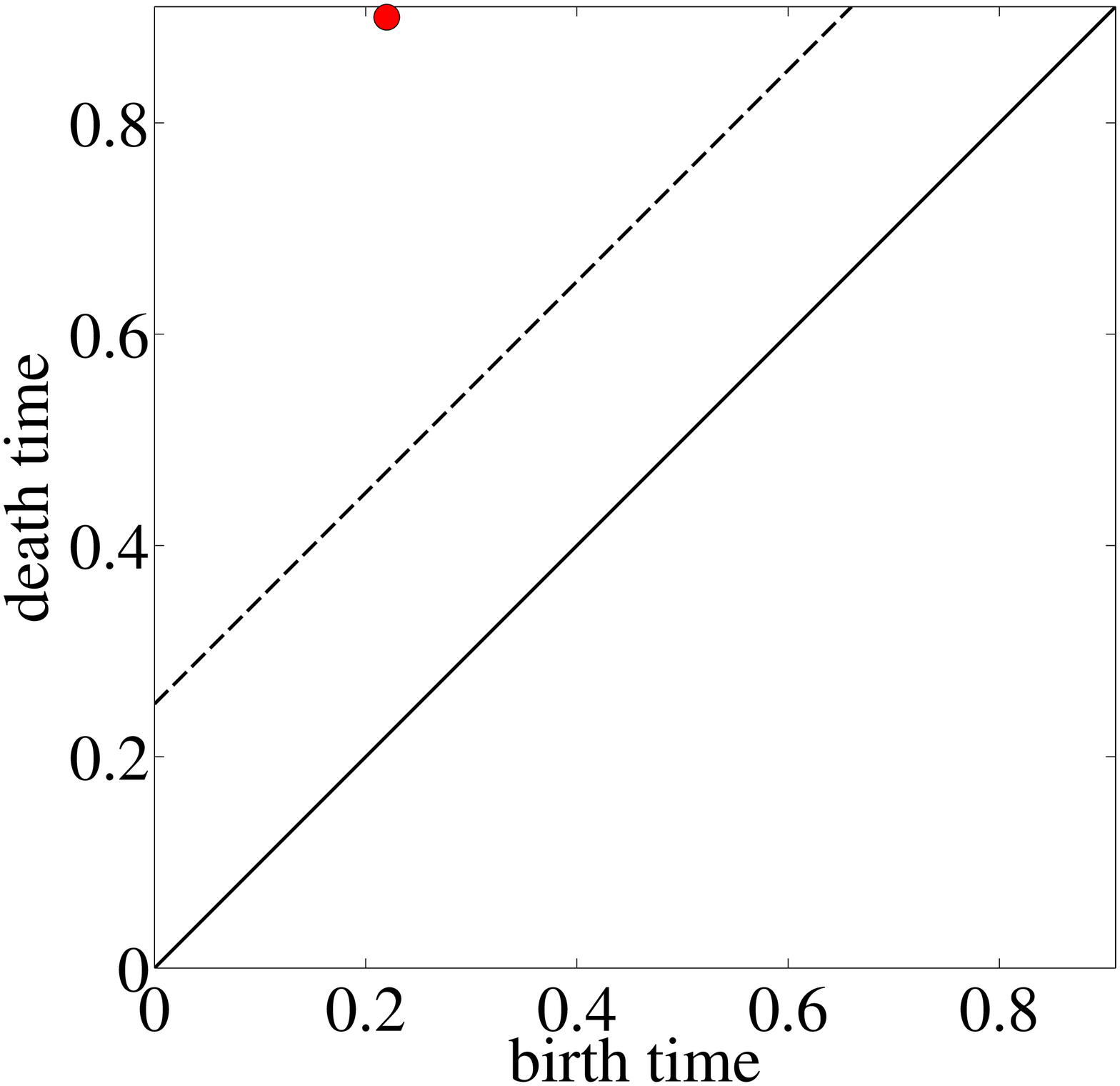}
  \centerline{\footnotesize{(c)}}\medskip
\end{minipage}
  \caption{(a) Noisy data samples from an underlying circle in 2-dimensional space, (b) persistence diagram for $H_0$, (c) persistence diagram for $H_1$.}
  \label{fig:pers_circle}
\end{figure}
As an example, let us consider the point cloud $\mathcal{Z}$ shown in Figure \ref{fig:pers_circle}(a), with noisy samples drawn from a circle, which has Betti numbers $\beta_0 = 1$, $\beta_1 = 1$, and $\beta_k = 0$ for $k > 1$.  At the value $\epsilon = 0$, the simplicial complex that is formed from $\mathcal{Z}$ is a collection of all the individual points not connected to any other point, resulting in the birth of $n$ topological features in the $H_0$ persistence diagram shown in Figure \ref{fig:pers_circle}(b).  As the scale increases, all of these little features die and only one persists until the largest scale under consideration; thus the persistent $\beta_0 = 1$.  Looking at the $H_1$ persistence diagram in Figure \ref{fig:pers_circle}(c), we see that the only feature that is born persists until the largest scale and thus the persistent $\beta_1 = 1$.  It is born at approximately a scale parameter of 0.2, which is when all of the points have been connected into a ring in the simplicial complex.

\section{Constructing Higher Order Simplices from Bipartite Graphs}
\label{sec:higher_order_simplices}

\begin{figure}
\begin{minipage}[c]{0.49\linewidth}
  \centering
  \includegraphics[width=3.5cm]{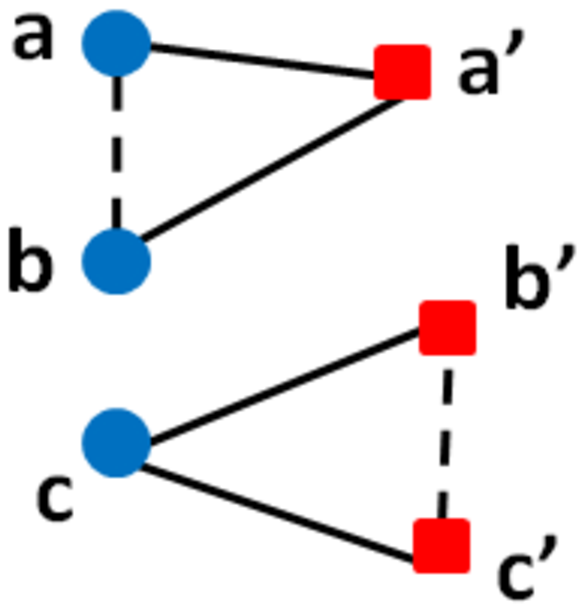}
  \centerline{\footnotesize{(a)}}\medskip
\end{minipage}
\begin{minipage}[c]{0.49\linewidth}
  \centering
  \includegraphics[width=3.5cm]{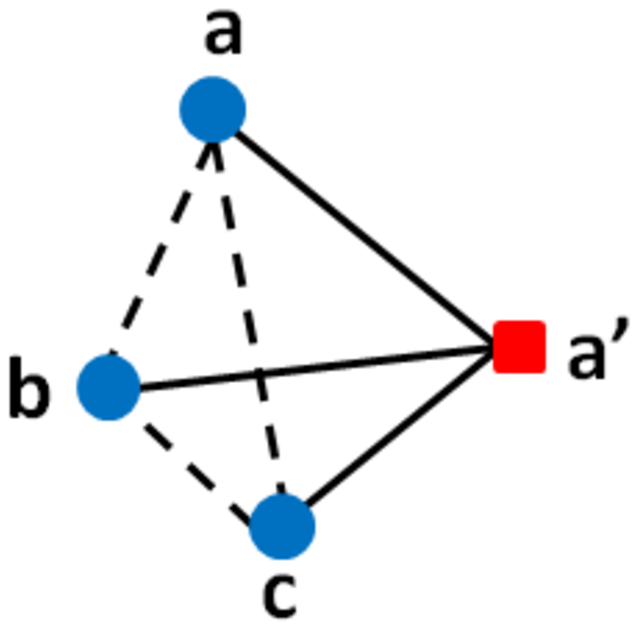}
  \centerline{\footnotesize{(b)}}\medskip
\end{minipage}
  \caption{(a) A simplicial complex with two 2-simplices from a bipartite graph between circle and square classes generated using length-2 walks (dotted lines), (b) a complex created with one 3-simplex using the same approach.}
  \label{fig:graphwalk}
\end{figure}
Consider the examples in Figure \ref{fig:graphwalk}. In the first example, we start with three points in a two-dimensional space where all points are within $\epsilon$ of each other.  Two share a class label and are thus not initially connected by an edge.  The initial graph $\mathbf{A}$ has two line segment simplices.  After including the graph walk, an intraclass edge is introduced.  Now $\tilde{\mathbf{A}}$ has a triangle simplex.  The second example is similar, but has four points in three-dimensional space, with three of the four points sharing a class label.  Here we form a tetrahedron after introducing the length two graph walk edges.

\section{Necessity of Decision Boundary Topology}
\label{sec:db_topo_necessity}
When understanding the decision surfaces of labeled data, using the approach presented in our paper is more accurate than using the unlabeled data topology for either of the classes as pursued in \cite{BianchiniS2014, GussS2018}. We use a simple counter-example in Figure \ref{fig:counter_example} to demonstrate this. Clearly, neither of the classes reflect the true topology of the decision boundary.

\begin{figure}
  \centering
  \includegraphics[width=6cm]{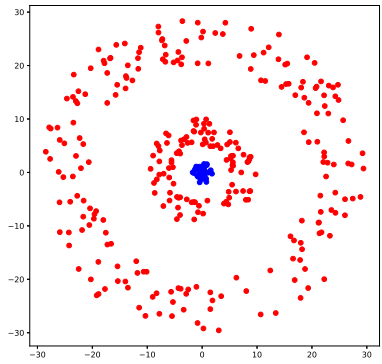}
  \caption{The \textit{blue} class has Betti numbers $\beta_0 = 1, \beta_1 = 0$. The \textit{red} class has Betti numbers $\beta_0 = 2, \beta_1 = 2$. Neither of them reflect the true topology of the decision boundary which has Betti numbers $\beta_0 = 1, \beta_1 = 1$}
  \label{fig:counter_example}
\end{figure}

\section{Implementation Notes}
\label{sec:impl_notes}
We adopt several approaches to make our implementations efficient. We will provide describe them briefly:
\begin{itemize}
\item We use $\epsilon-$neighborhood graphs to compute the LVR complexes, but to limit the number of simplices, we restrict the number of nearest neighbors for any point to $20$. We then symmetrize the graph and use it for obtaining the P-LVR and LS-LVR constructions.
\item The neighborhood graphs are computed efficiently using Cython code interfaced to the main Python package that we developed.
\item We estimate the distance matrices for the LVR constructions and use the efficient Ripser (\url{https://github.com/Ripser/ripser}) to obtain the persistence diagrams.
\item The entire pipeline (neighborhood graph construction and LVR estimation) to estimate the Betti numbers $\beta_0$ and $\beta_1$ for two classes runs in less than $1$ minute for about $1000$ points per class (the standard size of our test datasets). The program runs in a single core using less than 500MB of RAM in a standard computer.
\end{itemize}

\section{Decision Boundary Complexes}
We will display the decision boundary complexes corresponding to the demonstration in Section \ref{sec:demo_homology_rec}. The local scale multipliers for the LS-LVR filtration are varied between $0.5$ and $1.5$ in 100 increments. For the P-LVR filtration, the scale parameters are varied between $0$ and $10$ in 100 increments. We only show $20$ complexes for each filtration in evenly spaced increments.

\begin{figure}[h]
\begin{minipage}[c]{0.19\linewidth}
  \centering
  \includegraphics[width=2.8cm]{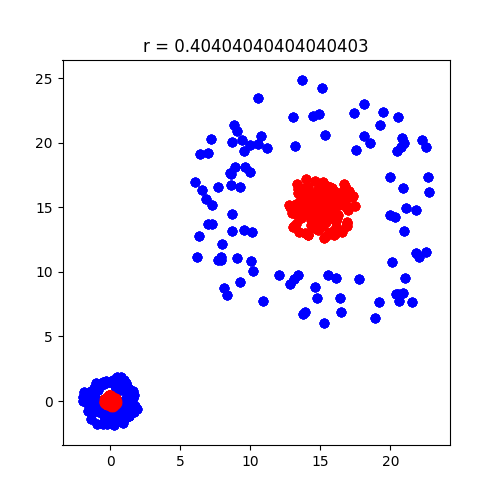}
\end{minipage}
\begin{minipage}[c]{0.19\linewidth}
  \centering
  \includegraphics[width=2.8cm]{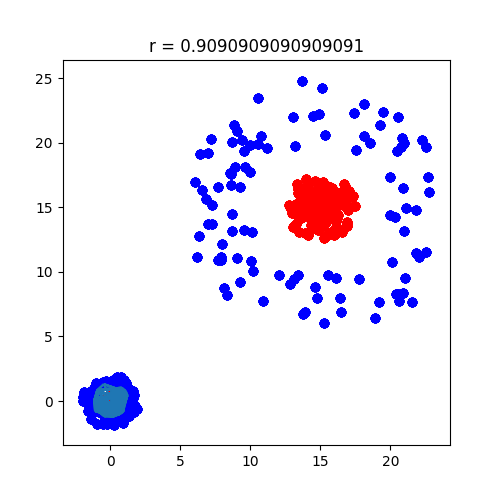}
\end{minipage}
\begin{minipage}[c]{0.19\linewidth}
  \centering
  \includegraphics[width=2.8cm]{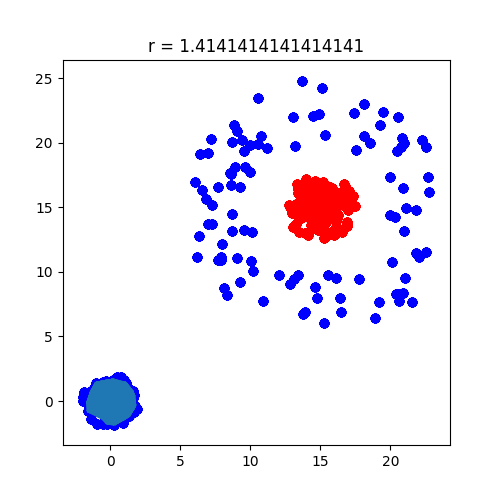}
\end{minipage}
\begin{minipage}[c]{0.19\linewidth}
  \centering
  \includegraphics[width=2.8cm]{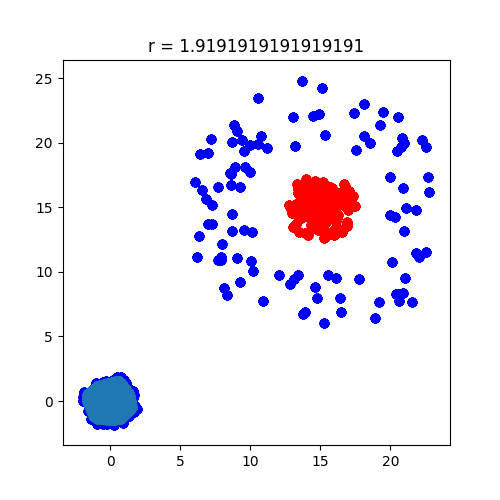}
\end{minipage}
\begin{minipage}[c]{0.19\linewidth}
  \centering
  \includegraphics[width=2.8cm]{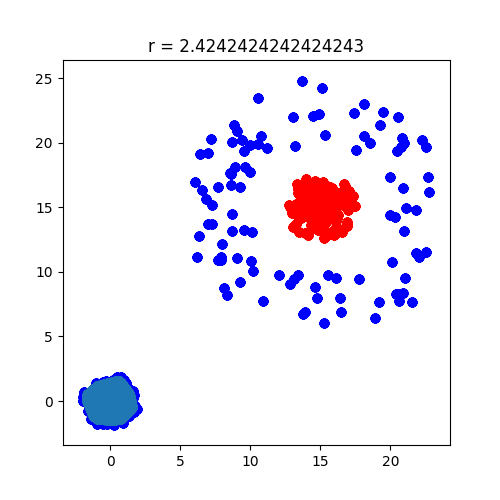}
\end{minipage}

\begin{minipage}[c]{0.19\linewidth}
  \centering
  \includegraphics[width=2.8cm]{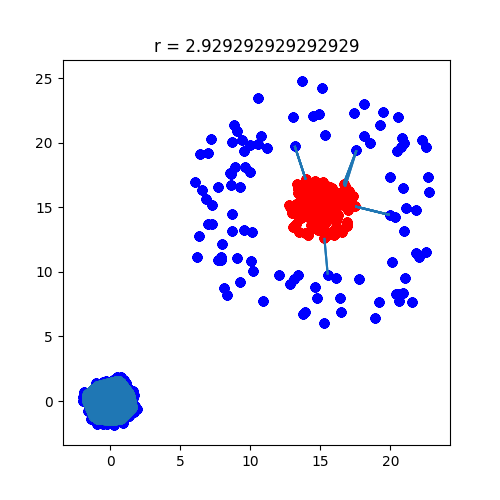}
\end{minipage}
\begin{minipage}[c]{0.19\linewidth}
  \centering
  \includegraphics[width=2.8cm]{Figures/eps_complexes_25.png}
\end{minipage}
\begin{minipage}[c]{0.19\linewidth}
  \centering
  \includegraphics[width=2.8cm]{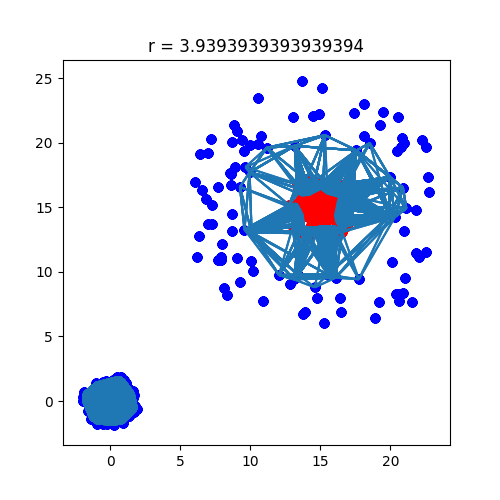}
\end{minipage}
\begin{minipage}[c]{0.19\linewidth}
  \centering
  \includegraphics[width=2.8cm]{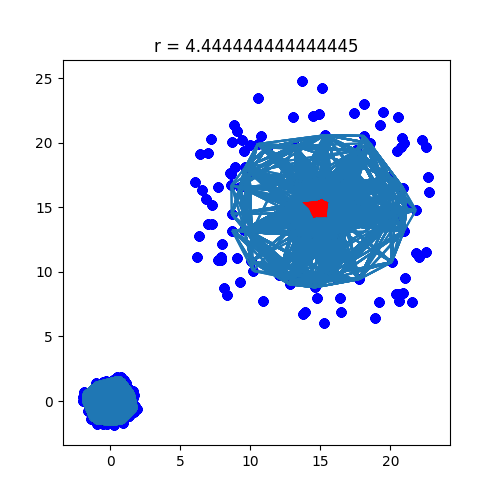}
\end{minipage}
\begin{minipage}[c]{0.19\linewidth}
  \centering
  \includegraphics[width=2.8cm]{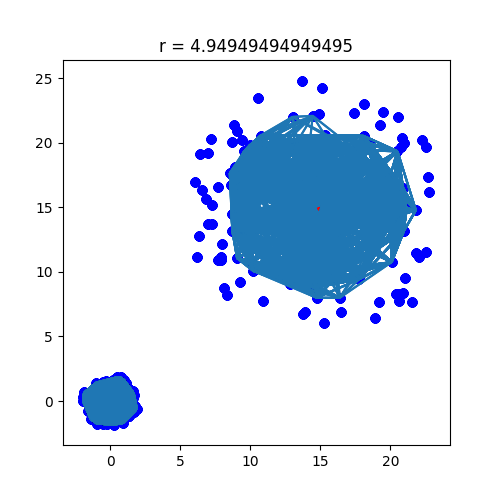}
\end{minipage}

\begin{minipage}[c]{0.19\linewidth}
  \centering
  \includegraphics[width=2.8cm]{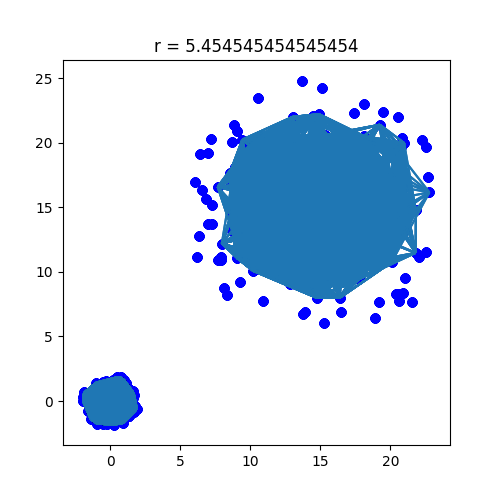}
\end{minipage}
\begin{minipage}[c]{0.19\linewidth}
  \centering
  \includegraphics[width=2.8cm]{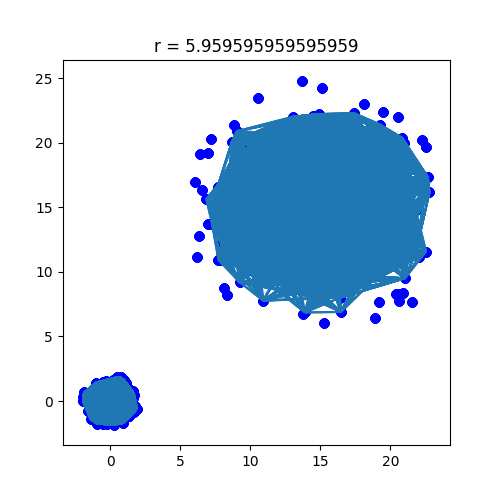}
\end{minipage}
\begin{minipage}[c]{0.19\linewidth}
  \centering
  \includegraphics[width=2.8cm]{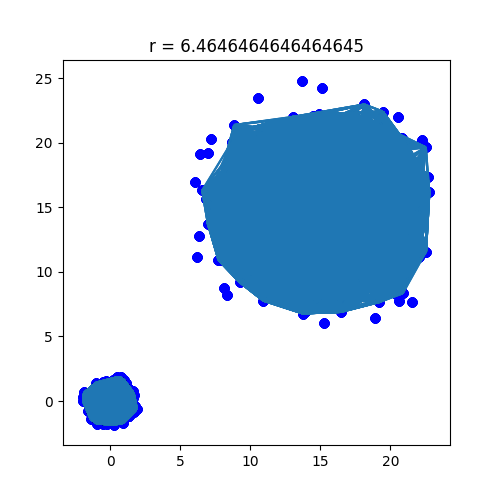}
\end{minipage}
\begin{minipage}[c]{0.19\linewidth}
  \centering
  \includegraphics[width=2.8cm]{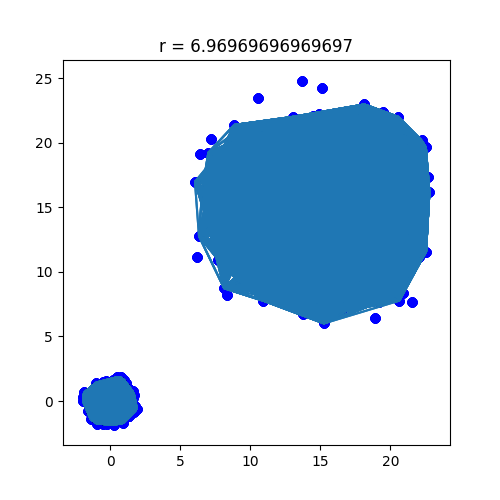}
\end{minipage}
\begin{minipage}[c]{0.19\linewidth}
  \centering
  \includegraphics[width=2.8cm]{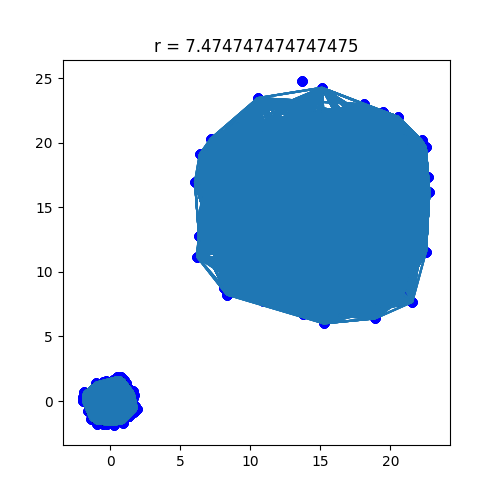}
\end{minipage}

\begin{minipage}[c]{0.19\linewidth}
  \centering
  \includegraphics[width=2.8cm]{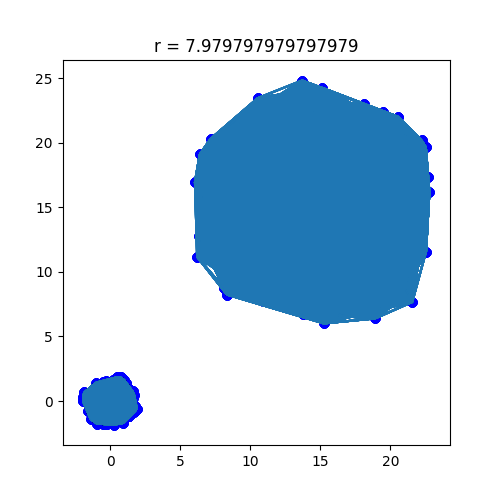}
\end{minipage}
\begin{minipage}[c]{0.19\linewidth}
  \centering
  \includegraphics[width=2.8cm]{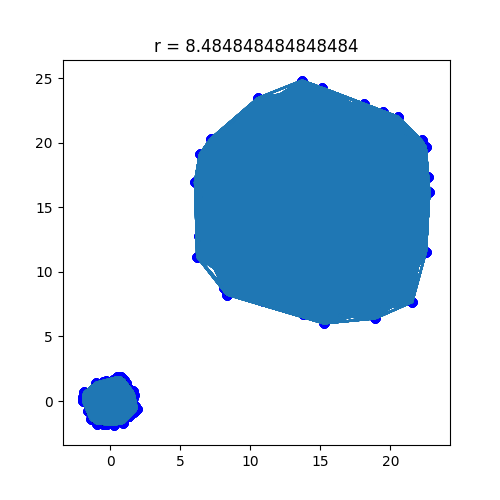}
\end{minipage}
\begin{minipage}[c]{0.19\linewidth}
  \centering
  \includegraphics[width=2.8cm]{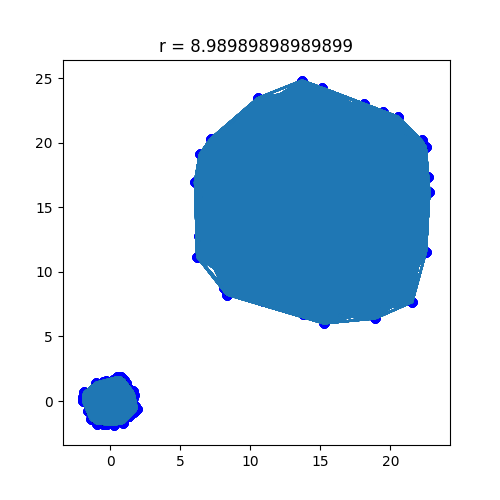}
\end{minipage}
\begin{minipage}[c]{0.19\linewidth}
  \centering
  \includegraphics[width=2.8cm]{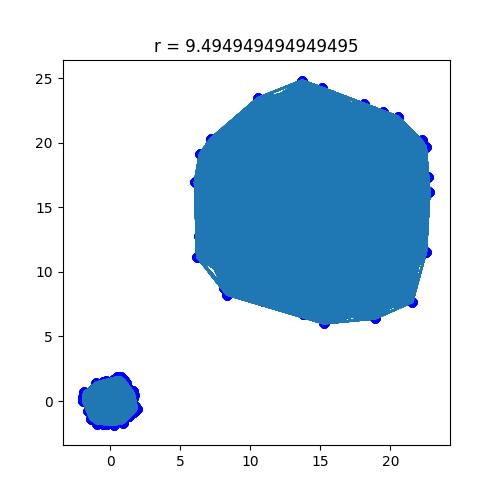}
\end{minipage}
\begin{minipage}[c]{0.19\linewidth}
  \centering
  \includegraphics[width=2.8cm]{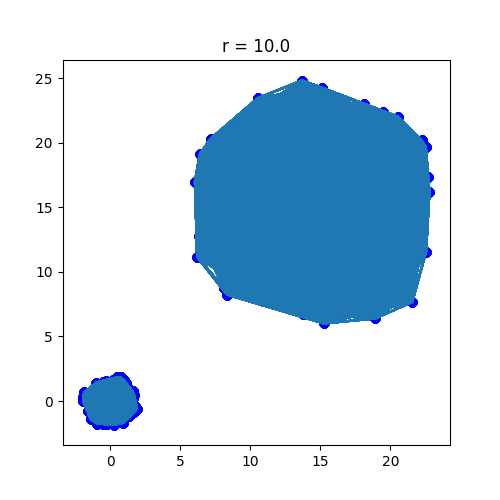}
\end{minipage}
  \caption{P-LVR complexes - the scales are given in the title of each image.}
  \label{fig:P_LVR_eps}
\end{figure}

\clearpage

\begin{figure}[h]
\begin{minipage}[c]{0.19\linewidth}
  \centering
  \includegraphics[width=2.8cm]{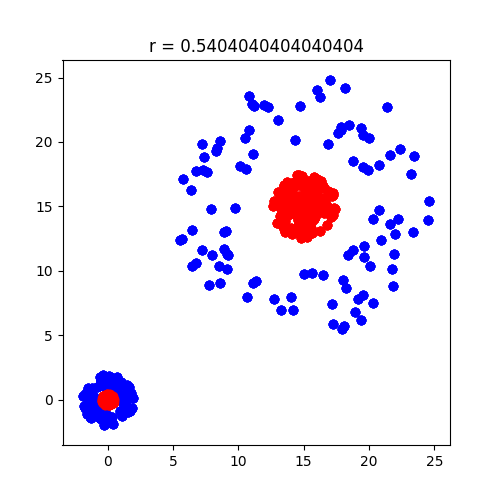}
\end{minipage}
\begin{minipage}[c]{0.19\linewidth}
  \centering
  \includegraphics[width=2.8cm]{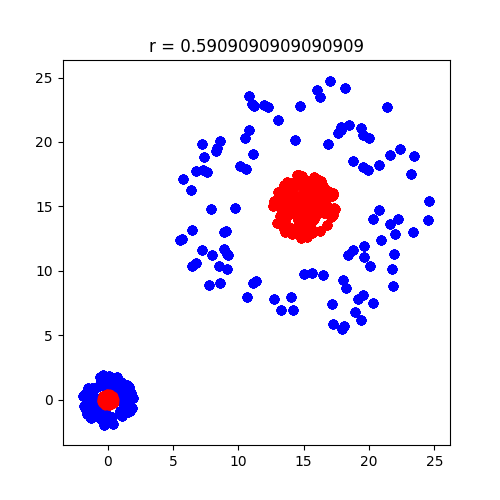}
\end{minipage}
\begin{minipage}[c]{0.19\linewidth}
  \centering
  \includegraphics[width=2.8cm]{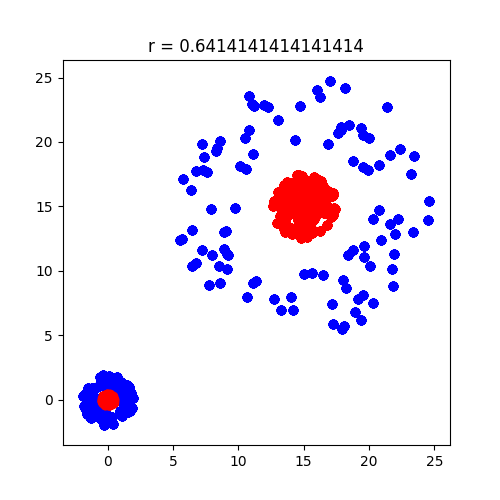}
\end{minipage}
\begin{minipage}[c]{0.19\linewidth}
  \centering
  \includegraphics[width=2.8cm]{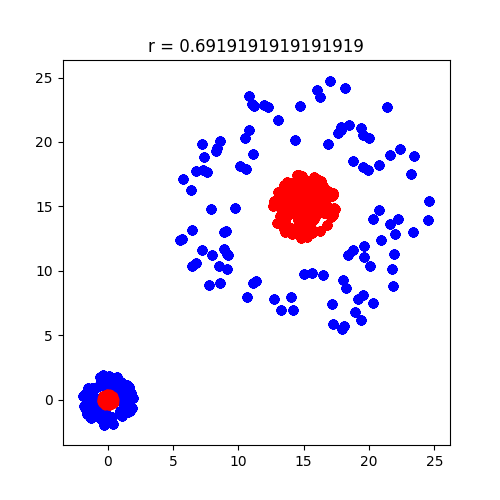}
\end{minipage}
\begin{minipage}[c]{0.19\linewidth}
  \centering
  \includegraphics[width=2.8cm]{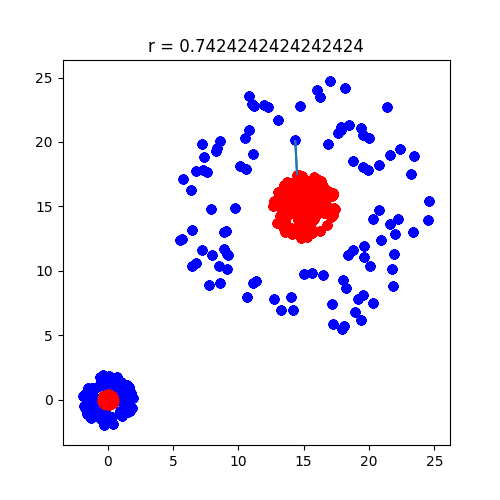}
\end{minipage}

\begin{minipage}[c]{0.19\linewidth}
  \centering
  \includegraphics[width=2.8cm]{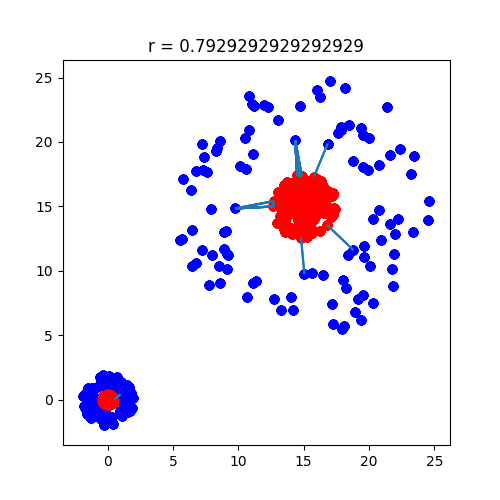}
\end{minipage}
\begin{minipage}[c]{0.19\linewidth}
  \centering
  \includegraphics[width=2.8cm]{Figures/knn_rho_complexes_25.png}
\end{minipage}
\begin{minipage}[c]{0.19\linewidth}
  \centering
  \includegraphics[width=2.8cm]{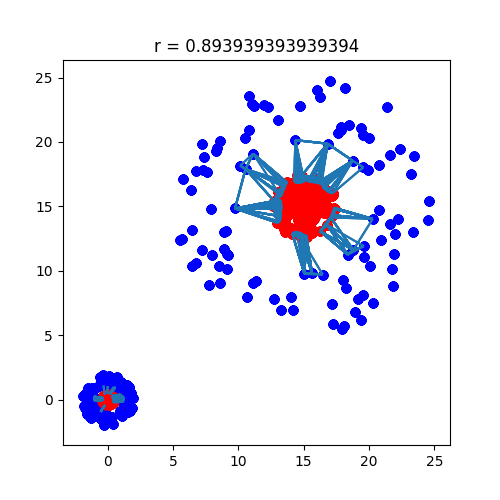}
\end{minipage}
\begin{minipage}[c]{0.19\linewidth}
  \centering
  \includegraphics[width=2.8cm]{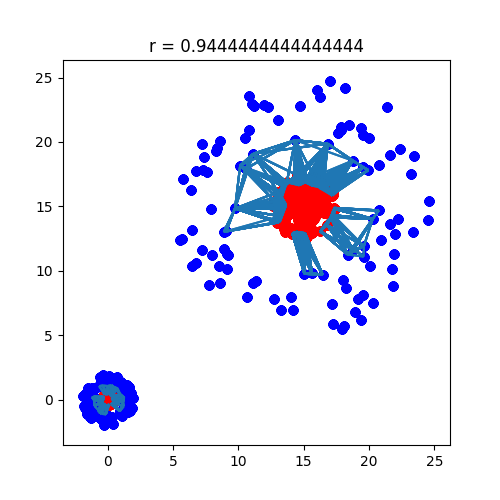}
\end{minipage}
\begin{minipage}[c]{0.19\linewidth}
  \centering
  \includegraphics[width=2.8cm]{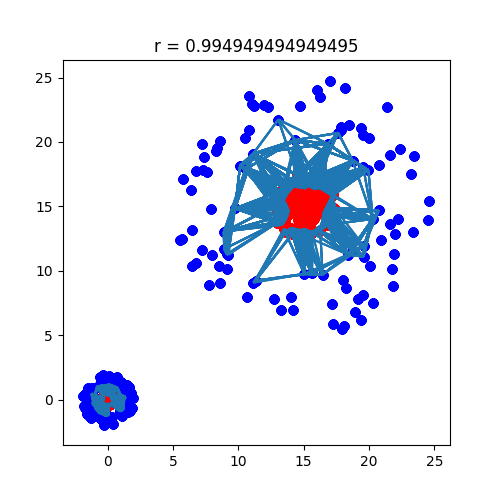}
\end{minipage}

\begin{minipage}[c]{0.19\linewidth}
  \centering
  \includegraphics[width=2.8cm]{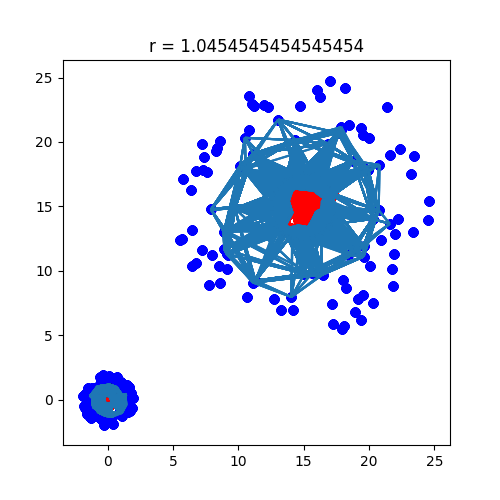}
\end{minipage}
\begin{minipage}[c]{0.19\linewidth}
  \centering
  \includegraphics[width=2.8cm]{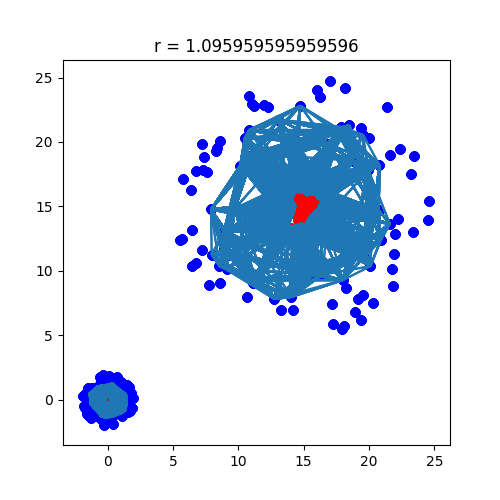}
\end{minipage}
\begin{minipage}[c]{0.19\linewidth}
  \centering
  \includegraphics[width=2.8cm]{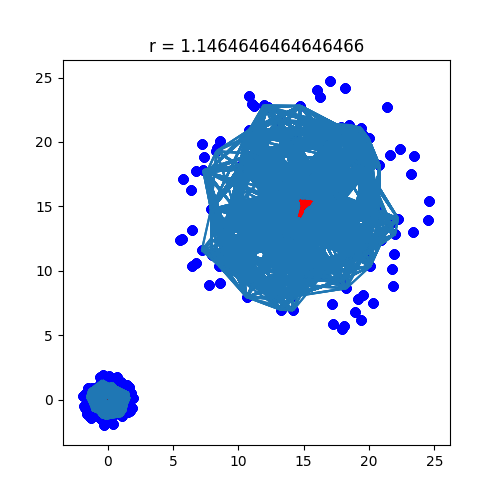}
\end{minipage}
\begin{minipage}[c]{0.19\linewidth}
  \centering
  \includegraphics[width=2.8cm]{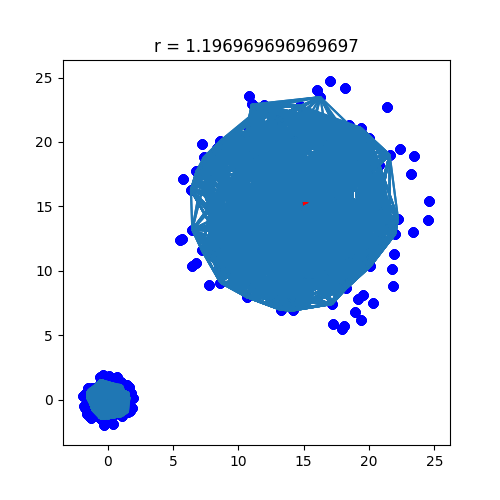}
\end{minipage}
\begin{minipage}[c]{0.19\linewidth}
  \centering
  \includegraphics[width=2.8cm]{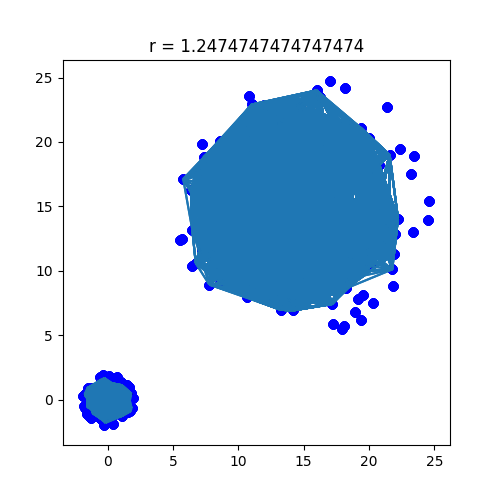}
\end{minipage}

\begin{minipage}[c]{0.19\linewidth}
  \centering
  \includegraphics[width=2.8cm]{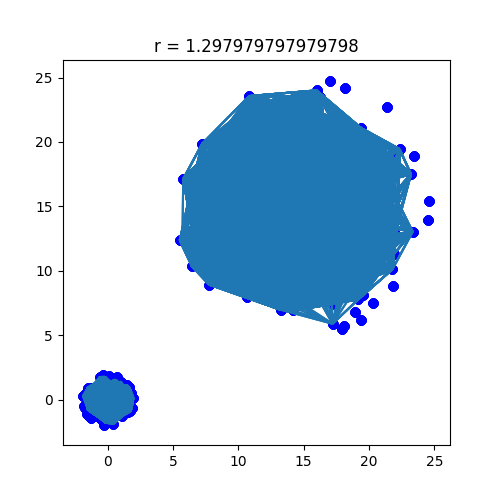}
\end{minipage}
\begin{minipage}[c]{0.19\linewidth}
  \centering
  \includegraphics[width=2.8cm]{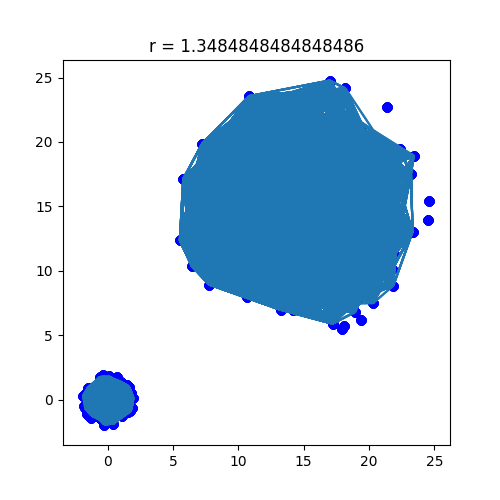}
\end{minipage}
\begin{minipage}[c]{0.19\linewidth}
  \centering
  \includegraphics[width=2.8cm]{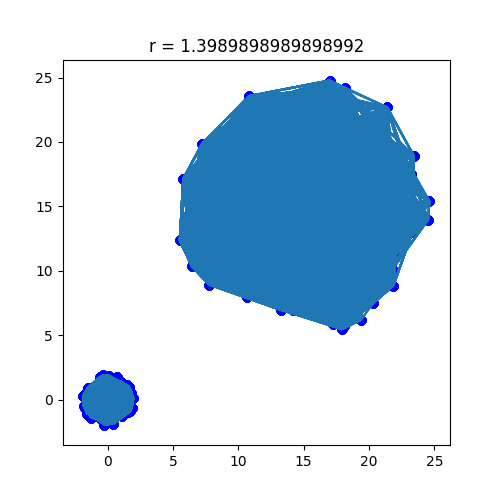}
\end{minipage}
\begin{minipage}[c]{0.19\linewidth}
  \centering
  \includegraphics[width=2.8cm]{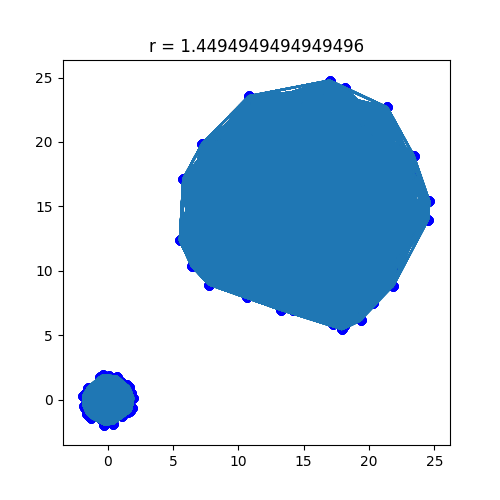}
\end{minipage}
\begin{minipage}[c]{0.19\linewidth}
  \centering
  \includegraphics[width=2.8cm]{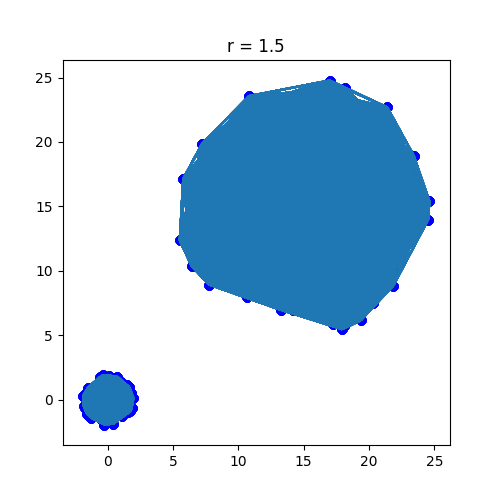}
\end{minipage}
  \caption{LS-LVR complexes - the local scale multipliers are given in the title of each image.}
  \label{fig:LS_LVR_knn_rho}
\end{figure}

\section{Data and Model Complexities: Total Lifetimes of Homology Groups}
\label{sec:data_model_complexities}
We provide the data model complexities for all binary datasets and trained classifiers for the three applications.

\begin{table}[h]
\centering
\small
\caption{Data complexity for MNIST: Total life of H0 groups}
\begin{tabular}{rllllllllll}
\hline
    & 0   & 1     & 2     & 3     & 4     & 5      & 6     & 7     & 8      & 9      \\
\hline
  0 &     & 153.0 & 543.0 & 607.0 & 388.0 & 758.0  & 734.0 & 271.0 & 634.0  & 413.0  \\
  1 &     &       & 342.0 & 262.0 & 287.0 & 212.0  & 241.0 & 340.0 & 317.0  & 219.0  \\
  2 &     &       &       & 951.0 & 564.0 & 490.0  & 532.0 & 672.0 & 805.0  & 549.0  \\
  3 &     &       &       &       & 381.0 & 1292.0 & 454.0 & 493.0 & 1111.0 & 553.0  \\
  4 &     &       &       &       &       & 488.0  & 582.0 & 917.0 & 594.0  & 1682.0 \\
  5 &     &       &       &       &       &        & 548.0 & 419.0 & 1248.0 & 650.0  \\
  6 &     &       &       &       &       &        &       & 242.0 & 555.0  & 388.0  \\
  7 &     &       &       &       &       &        &       &       & 522.0  & 1283.0 \\
  8 &     &       &       &       &       &        &       &       &        & 731.0  \\
  9 &     &       &       &       &       &        &       &       &        &        \\
\hline
\end{tabular}
\end{table}
\begin{table}[h]
\centering
\small
\caption{Data complexity for MNIST: Total life of H1 groups}
\begin{tabular}{rllllllllll}
\hline
    & 0   & 1   & 2     & 3     & 4     & 5     & 6     & 7     & 8     & 9     \\
\hline
  0 &     & 9.0 & 205.0 & 202.0 & 91.0  & 300.0 & 229.0 & 78.0  & 211.0 & 112.0 \\
  1 &     &     & 72.0  & 40.0  & 73.0  & 24.0  & 79.0  & 60.0  & 78.0  & 37.0  \\
  2 &     &     &       & 400.0 & 211.0 & 183.0 & 216.0 & 228.0 & 350.0 & 178.0 \\
  3 &     &     &       &       & 178.0 & 542.0 & 135.0 & 164.0 & 480.0 & 203.0 \\
  4 &     &     &       &       &       & 192.0 & 166.0 & 350.0 & 263.0 & 714.0 \\
  5 &     &     &       &       &       &       & 195.0 & 143.0 & 536.0 & 237.0 \\
  6 &     &     &       &       &       &       &       & 60.0  & 210.0 & 121.0 \\
  7 &     &     &       &       &       &       &       &       & 204.0 & 459.0 \\
  8 &     &     &       &       &       &       &       &       &       & 297.0 \\
  9 &     &     &       &       &       &       &       &       &       &       \\
\hline
\end{tabular}
\end{table}
\begin{table}[h]
\centering
\small
\caption{Model complexity for MNIST: Total life of H0 groups}
\begin{tabular}{rllllllllll}
\hline
    & 0   & 1     & 2     & 3     & 4     & 5      & 6     & 7     & 8      & 9      \\
\hline
  0 &     & 153.0 & 543.0 & 607.0 & 388.0 & 758.0  & 734.0 & 271.0 & 634.0  & 413.0  \\
  1 &     &       & 342.0 & 262.0 & 287.0 & 212.0  & 241.0 & 340.0 & 317.0  & 219.0  \\
  2 &     &       &       & 951.0 & 564.0 & 490.0  & 532.0 & 672.0 & 805.0  & 549.0  \\
  3 &     &       &       &       & 381.0 & 1292.0 & 454.0 & 493.0 & 1111.0 & 553.0  \\
  4 &     &       &       &       &       & 488.0  & 582.0 & 917.0 & 594.0  & 1682.0 \\
  5 &     &       &       &       &       &        & 548.0 & 419.0 & 1248.0 & 650.0  \\
  6 &     &       &       &       &       &        &       & 242.0 & 555.0  & 388.0  \\
  7 &     &       &       &       &       &        &       &       & 522.0  & 1283.0 \\
  8 &     &       &       &       &       &        &       &       &        & 731.0  \\
  9 &     &       &       &       &       &        &       &       &        &        \\
\hline
\end{tabular}
\end{table}
\begin{table}[h]
\centering
\small
\caption{Model complexity for MNIST: Total life of H1 groups}
\begin{tabular}{rllllllllll}
\hline
    & 0   & 1   & 2     & 3     & 4     & 5     & 6     & 7     & 8     & 9     \\
\hline
  0 &     & 9.0 & 205.0 & 202.0 & 91.0  & 300.0 & 229.0 & 78.0  & 211.0 & 112.0 \\
  1 &     &     & 72.0  & 40.0  & 73.0  & 24.0  & 79.0  & 60.0  & 78.0  & 37.0  \\
  2 &     &     &       & 400.0 & 211.0 & 183.0 & 216.0 & 228.0 & 350.0 & 178.0 \\
  3 &     &     &       &       & 178.0 & 542.0 & 135.0 & 164.0 & 480.0 & 203.0 \\
  4 &     &     &       &       &       & 192.0 & 166.0 & 350.0 & 263.0 & 714.0 \\
  5 &     &     &       &       &       &       & 195.0 & 143.0 & 536.0 & 237.0 \\
  6 &     &     &       &       &       &       &       & 60.0  & 210.0 & 121.0 \\
  7 &     &     &       &       &       &       &       &       & 204.0 & 459.0 \\
  8 &     &     &       &       &       &       &       &       &       & 297.0 \\
  9 &     &     &       &       &       &       &       &       &       &       \\
\hline
\end{tabular}
\end{table}
\begin{table}[h]
\centering
\small
\caption{Data complexity for FashionMNIST: Total life of H0 groups}
\begin{tabular}{lllllllllll}
\hline
             & T-shirt/top   & Trouser   & Pullover   & Dress   & Coat   & Sandal   & Shirt   & Sneaker   & Bag   & Ankle boot   \\
\hline
 T-shirt/top &               & 426.0     & 698.0      & 975.0   & 539.0  & 206.0    & 1946.0  & 122.0     & 465.0 & 173.0        \\
 Trouser     &               &           & 273.0      & 668.0   & 300.0  & 96.0     & 410.0   & 79.0      & 181.0 & 88.0         \\
 Pullover    &               &           &            & 488.0   & 2497.0 & 138.0    & 2108.0  & 98.0      & 390.0 & 126.0        \\
 Dress       &               &           &            &         & 775.0  & 125.0    & 980.0   & 92.0      & 310.0 & 128.0        \\
 Coat        &               &           &            &         &        & 124.0    & 1792.0  & 108.0     & 434.0 & 121.0        \\
 Sandal      &               &           &            &         &        &          & 149.0   & 724.0     & 314.0 & 554.0        \\
 Shirt       &               &           &            &         &        &          &         & 100.0     & 500.0 & 151.0        \\
 Sneaker     &               &           &            &         &        &          &         &           & 272.0 & 735.0        \\
 Bag         &               &           &            &         &        &          &         &           &       & 287.0        \\
 Ankle boot  &               &           &            &         &        &          &         &           &       &              \\
\hline
\end{tabular}
\end{table}
\begin{table}[h]
\centering
\small
\caption{Data complexity for FashionMNIST: Total life of H1 groups}
\begin{tabular}{lllllllllll}
\hline
             & T-shirt/top   & Trouser   & Pullover   & Dress   & Coat   & Sandal   & Shirt   & Sneaker   & Bag   & Ankle boot   \\
\hline
 T-shirt/top &               & 49.0      & 105.0      & 177.0   & 81.0   & 34.0     & 445.0   & 14.0      & 51.0  & 23.0         \\
 Trouser     &               &           & 22.0       & 100.0   & 31.0   & 9.0      & 44.0    &           & 8.0   & 5.0          \\
 Pullover    &               &           &            & 99.0    & 692.0  & 20.0     & 592.0   & 2.0       & 75.0  & 13.0         \\
 Dress       &               &           &            &         & 134.0  & 35.0     & 148.0   & 13.0      & 41.0  & 7.0          \\
 Coat        &               &           &            &         &        & 16.0     & 561.0   & 6.0       & 70.0  & 11.0         \\
 Sandal      &               &           &            &         &        &          & 23.0    & 156.0     & 30.0  & 92.0         \\
 Shirt       &               &           &            &         &        &          &         & 4.0       & 92.0  & 16.0         \\
 Sneaker     &               &           &            &         &        &          &         &           & 30.0  & 154.0        \\
 Bag         &               &           &            &         &        &          &         &           &       & 37.0         \\
 Ankle boot  &               &           &            &         &        &          &         &           &       &              \\
\hline
\end{tabular}
\end{table}
\begin{table}[h]
\centering
\small
\caption{Model complexity for FashionMNIST: Total life of H0 groups}
\begin{tabular}{lllllllllll}
\hline
             & T-shirt/top   & Trouser   & Pullover   & Dress   & Coat   & Sandal   & Shirt   & Sneaker   & Bag   & Ankle boot   \\
\hline
 T-shirt/top &               & 426.0     & 698.0      & 975.0   & 539.0  & 206.0    & 1946.0  & 122.0     & 465.0 & 173.0        \\
 Trouser     &               &           & 273.0      & 668.0   & 300.0  & 96.0     & 410.0   & 79.0      & 181.0 & 88.0         \\
 Pullover    &               &           &            & 488.0   & 2497.0 & 138.0    & 2108.0  & 98.0      & 390.0 & 126.0        \\
 Dress       &               &           &            &         & 775.0  & 125.0    & 980.0   & 92.0      & 310.0 & 128.0        \\
 Coat        &               &           &            &         &        & 124.0    & 1792.0  & 108.0     & 434.0 & 121.0        \\
 Sandal      &               &           &            &         &        &          & 149.0   & 724.0     & 314.0 & 554.0        \\
 Shirt       &               &           &            &         &        &          &         & 100.0     & 500.0 & 151.0        \\
 Sneaker     &               &           &            &         &        &          &         &           & 272.0 & 735.0        \\
 Bag         &               &           &            &         &        &          &         &           &       & 287.0        \\
 Ankle boot  &               &           &            &         &        &          &         &           &       &              \\
\hline
\end{tabular}
\end{table}
\begin{table}[h]
\centering
\small
\caption{Model complexity for FashionMNIST: Total life of H1 groups}
\begin{tabular}{lllllllllll}
\hline
             & T-shirt/top   & Trouser   & Pullover   & Dress   & Coat   & Sandal   & Shirt   & Sneaker   & Bag   & Ankle boot   \\
\hline
 T-shirt/top &               & 49.0      & 105.0      & 177.0   & 81.0   & 34.0     & 445.0   & 14.0      & 51.0  & 23.0         \\
 Trouser     &               &           & 22.0       & 100.0   & 31.0   & 9.0      & 44.0    &           & 8.0   & 5.0          \\
 Pullover    &               &           &            & 99.0    & 692.0  & 20.0     & 592.0   & 2.0       & 75.0  & 13.0         \\
 Dress       &               &           &            &         & 134.0  & 35.0     & 148.0   & 13.0      & 41.0  & 7.0          \\
 Coat        &               &           &            &         &        & 16.0     & 561.0   & 6.0       & 70.0  & 11.0         \\
 Sandal      &               &           &            &         &        &          & 23.0    & 156.0     & 30.0  & 92.0         \\
 Shirt       &               &           &            &         &        &          &         & 4.0       & 92.0  & 16.0         \\
 Sneaker     &               &           &            &         &        &          &         &           & 30.0  & 154.0        \\
 Bag         &               &           &            &         &        &          &         &           &       & 37.0         \\
 Ankle boot  &               &           &            &         &        &          &         &           &       &              \\
\hline
\end{tabular}
\end{table}
\begin{table}[h]
\centering
\small
\caption{Data complexity for CIFAR10: Total life of H0 groups}
\begin{tabular}{lllllllllll}
\hline
            & airplane   & automobile   & bird   & cat    & deer   & dog    & frog   & horse   & ship   & truck   \\
\hline
 airplane   &            & 990.0        & 1384.0 & 1019.0 & 1004.0 & 890.0  & 766.0  & 1097.0  & 1287.0 & 1048.0  \\
 automobile &            &              & 1009.0 & 1111.0 & 1011.0 & 888.0  & 956.0  & 1008.0  & 1129.0 & 1328.0  \\
 bird       &            &              &        & 1417.0 & 1431.0 & 1341.0 & 1144.0 & 1336.0  & 1090.0 & 973.0   \\
 cat        &            &              &        &        & 1307.0 & 1562.0 & 1281.0 & 1255.0  & 1069.0 & 1058.0  \\
 deer       &            &              &        &        &        & 1132.0 & 1101.0 & 1154.0  & 962.0  & 887.0   \\
 dog        &            &              &        &        &        &        & 1157.0 & 1383.0  & 987.0  & 852.0   \\
 frog       &            &              &        &        &        &        &        & 865.0   & 783.0  & 819.0   \\
 horse      &            &              &        &        &        &        &        &         & 1085.0 & 1132.0  \\
 ship       &            &              &        &        &        &        &        &         &        & 1201.0  \\
 truck      &            &              &        &        &        &        &        &         &        &         \\
\hline
\end{tabular}
\end{table}
\begin{table}[h]
\centering
\small
\caption{Data complexity for CIFAR10: Total life of H1 groups}
\begin{tabular}{lllllllllll}
\hline
            & airplane   & automobile   & bird   & cat   & deer   & dog   & frog   & horse   & ship   & truck   \\
\hline
 airplane   &            & 267.0        & 339.0  & 216.0 & 221.0  & 223.0 & 158.0  & 259.0   & 415.0  & 239.0   \\
 automobile &            &              & 266.0  & 322.0 & 239.0  & 266.0 & 254.0  & 324.0   & 353.0  & 460.0   \\
 bird       &            &              &        & 544.0 & 434.0  & 490.0 & 294.0  & 377.0   & 242.0  & 259.0   \\
 cat        &            &              &        &       & 389.0  & 480.0 & 289.0  & 399.0   & 302.0  & 305.0   \\
 deer       &            &              &        &       &        & 338.0 & 280.0  & 282.0   & 279.0  & 197.0   \\
 dog        &            &              &        &       &        &       & 259.0  & 390.0   & 288.0  & 239.0   \\
 frog       &            &              &        &       &        &       &        & 251.0   & 175.0  & 149.0   \\
 horse      &            &              &        &       &        &       &        &         & 303.0  & 301.0   \\
 ship       &            &              &        &       &        &       &        &         &        & 373.0   \\
 truck      &            &              &        &       &        &       &        &         &        &         \\
\hline
\end{tabular}
\end{table}
\begin{table}[h]
\centering
\small
\caption{Model complexity for CIFAR10: Total life of H0 groups}
\begin{tabular}{lllllllllll}
\hline
            & airplane   & automobile   & bird   & cat    & deer   & dog    & frog   & horse   & ship   & truck   \\
\hline
 airplane   &            & 990.0        & 1384.0 & 1019.0 & 1004.0 & 890.0  & 766.0  & 1097.0  & 1287.0 & 1048.0  \\
 automobile &            &              & 1009.0 & 1111.0 & 1011.0 & 888.0  & 956.0  & 1008.0  & 1129.0 & 1328.0  \\
 bird       &            &              &        & 1417.0 & 1431.0 & 1341.0 & 1144.0 & 1336.0  & 1090.0 & 973.0   \\
 cat        &            &              &        &        & 1307.0 & 1562.0 & 1281.0 & 1255.0  & 1069.0 & 1058.0  \\
 deer       &            &              &        &        &        & 1132.0 & 1101.0 & 1154.0  & 962.0  & 887.0   \\
 dog        &            &              &        &        &        &        & 1157.0 & 1383.0  & 987.0  & 852.0   \\
 frog       &            &              &        &        &        &        &        & 865.0   & 783.0  & 819.0   \\
 horse      &            &              &        &        &        &        &        &         & 1085.0 & 1132.0  \\
 ship       &            &              &        &        &        &        &        &         &        & 1201.0  \\
 truck      &            &              &        &        &        &        &        &         &        &         \\
\hline
\end{tabular}
\end{table}
\begin{table}[h]
\centering
\small
\caption{Model complexity for CIFAR10: Total life of H1 groups}
\begin{tabular}{lllllllllll}
\hline
            & airplane   & automobile   & bird   & cat   & deer   & dog   & frog   & horse   & ship   & truck   \\
\hline
 airplane   &            & 267.0        & 339.0  & 216.0 & 221.0  & 223.0 & 158.0  & 259.0   & 415.0  & 239.0   \\
 automobile &            &              & 266.0  & 322.0 & 239.0  & 266.0 & 254.0  & 324.0   & 353.0  & 460.0   \\
 bird       &            &              &        & 544.0 & 434.0  & 490.0 & 294.0  & 377.0   & 242.0  & 259.0   \\
 cat        &            &              &        &       & 389.0  & 480.0 & 289.0  & 399.0   & 302.0  & 305.0   \\
 deer       &            &              &        &       &        & 338.0 & 280.0  & 282.0   & 279.0  & 197.0   \\
 dog        &            &              &        &       &        &       & 259.0  & 390.0   & 288.0  & 239.0   \\
 frog       &            &              &        &       &        &       &        & 251.0   & 175.0  & 149.0   \\
 horse      &            &              &        &       &        &       &        &         & 303.0  & 301.0   \\
 ship       &            &              &        &       &        &       &        &         &        & 373.0   \\
 truck      &            &              &        &       &        &       &        &         &        &         \\
\hline
\end{tabular}
\end{table}

\end{document}